\documentclass[11pt]{article}

\usepackage{amsthm}
\usepackage{amssymb}
\usepackage{amsfonts}
\usepackage{cancel}
\usepackage{fullpage}
\usepackage{liyang}
\usepackage{framed}
\usepackage{natbib}
\usepackage{verbatim}
\usepackage{enumitem}
\usepackage{array}
\usepackage{multirow}%
\usepackage{afterpage}
\usepackage{caption}
\usepackage{subcaption}
\usepackage{setspace}
\usepackage{soul}

\newcommand{\anote}[1]{\footnote{{\bf \color{green}Anindya:} {#1}}}
\newcommand{\pnote}[1]{\footnote{{\bf \color{purple}Phil}: {#1}}}

\usepackage[normalem]{ulem}

\usepackage{caption} 

\usepackage{todonotes}

\makeatletter
\newtheorem*{rep@theorem}{\rep@title}
\newcommand{\newreptheorem}[2]{
\newenvironment{rep#1}[1]{
 \def\rep@title{#2 \ref{##1}}
 \begin{rep@theorem}\itshape}
 {\end{rep@theorem}}}
\makeatother
\makeatletter
\newenvironment{proofof}[1]{\par
  \pushQED{\qed}%
  \normalfont \topsep6\p@\@plus6\p@\relax
  \trivlist
  \item[\hskip\labelsep
\emph{    Proof of #1\@addpunct{.}}]\ignorespaces
}{%
  \popQED\endtrivlist\@endpefalse
}
\makeatother

\def\colorful{1}
\ifnum\colorful=1

\newcommand{\blue}[1]{{{#1}}}
\newcommand{\red}[1]{{\color{red} {#1}}}
\newcommand{\green}[1]{{\color{green} {#1}}}

\fi
\ifnum\colorful=0

\newcommand{\blue}[1]{{{#1}}}
\newcommand{\red}[1]{{{#1}}}
\newcommand{\green}[1]{{{#1}}}

\fi

\usepackage{boxedminipage}

\newtheorem{theorem}{Theorem}
\newtheorem*{theorem*}{Theorem}
\newtheorem{lemma}[theorem]{Lemma}
\newtheorem{remark}[theorem]{Remark}
\newtheorem{proposition}[theorem]{Proposition}
\newtheorem{corollary}[theorem]{Corollary}
\newtheorem{definition}[theorem]{Definition}
\newtheorem*{noclaim*}{Claim}

\newtheorem{claim}[theorem]{Claim}
\newtheorem{fact}[theorem]{Fact}
\newtheorem{observation}[theorem]{Observation}

\newcommand{\CSI}{{\mathcal C}_{\mathrm{SI}}}
\newcommand{\si}{{\mathrm{SI}}}
\renewcommand{\ss}{\subseteq}
\newcommand{\cube}{{\mathrm{cube}}}
\newcommand{\cF}{{\cal F}}
\newcommand{\cC}{{\cal C}}

\newcommand{\tSigma}{\tilde{\Sigma}}
\newcommand{\tmu}{\tilde{\mu}}

\title{Density estimation
                   for shift-invariant multidimensional 
                   distributions}
\usepackage{times}

\author{
Anindya De\thanks{Supported by NSF grant CCF-1814706}\\
Northwestern University\\
{\tt anindya@eecs.northwestern.edu}
\and
Philip M. Long\\
Google\\
{\tt plong@google.com}
\and
\and Rocco A.~Servedio\thanks{Supported by NSF grants CCF-1319788 and CCF-1420349}\\
Columbia University \\
{\tt rocco@cs.columbia.edu}
}

\date{}

\begin{document}

\maketitle

\begin{abstract}
We study density estimation for  classes of \emph{shift-invariant} distributions over $\mathbb{R}^d$.  A multidimensional distribution is ``shift-invariant'' if, roughly  speaking, it is close in total variation distance to a small shift of it in any direction.  
Shift-invariance relaxes smoothness assumptions commonly used in
non-parametric density estimation to allow jump discontinuities.  
The different classes of distributions that we consider correspond to different rates of tail decay.

For each such class we give an efficient algorithm that learns any distribution in the class from independent samples with respect to total variation distance.  As a special case of our general result, we show that $d$-dimensional shift-invariant distributions which satisfy an exponential tail bound can be learned to total variation distance error $\eps$ using $\tilde{O}_d(1/ \epsilon^{d+2})$ examples and $\tilde{O}_d(1/ \epsilon^{2d+2})$ time.  This implies that, for constant $d$, multivariate log-concave distributions can be learned in $\tilde{O}_d(1/\epsilon^{2d+2})$ time using $\tilde{O}_d(1/\epsilon^{d+2})$ samples, answering a question of  \citep{diakonikolas2016learning}. 
All of our results extend to a model of \emph{noise-tolerant} density estimation using Huber's contamination model, in which the target distribution to be learned is a $(1-\eps,\eps)$ mixture of some unknown distribution in the class with some other arbitrary and unknown distribution, and the learning algorithm must output a hypothesis distribution with total variation distance error $O(\eps)$ from the target distribution.  We show that our general results are close to best possible by proving a simple $\Omega\left(1/\epsilon^d\right)$ information-theoretic lower bound on sample complexity even for learning bounded distributions that are shift-invariant.
\end{abstract}



\section{Introduction}

In multidimensional density estimation, an algorithm has access to independent draws from an unknown 
target probability distribution over $\mathbb{R}^d$, which is typically assumed to belong to or be close to some class of
``nice'' distributions.  The goal is to output a hypothesis distribution which with high probability is
close to the target distribution.  A number of different distance measures can be used to capture the notion of
closeness; in this work we use the total variation distance (also known as the ``statistical distance'' 
and equivalent to the $L_1$ distance).  This is a well studied framework which has been investigated in detail, see e.g. the books~\citep{DG85,devroye2012combinatorial}.

Multidimensional density estimation 
is typically attacked in one of
two ways.  In the first general approach a parameterized hypothesis class is chosen, and a setting of parameters
is chosen based on the observed data points.  This approach is justified given the belief that the parameterized class contains
a good approximation to the distribution generating the
data, or even that the parameterized class actually contains the target distribution.  See \citep{Dasgupta:99,KMV:10,MoitraValiant:10} for some well-known multidimensional distribution learning results in this line.  

In the second general approach a hypothesis distribution is constructed by ``smoothing'' the empirical distribution with a kernel function.  This approach is justified by the belief that the target distribution satisfies some smoothness assumptions, and is more appropriate when studying distributions that do not have a parametric representation.  The current paper falls within this second strand.

The most popular smoothness assumption is that the distribution has
a density that belongs to a Sobolev space 
\citep{sobolev1963theorem,barron1991minimum,holmstrom1992asymptotic,devroye2012combinatorial}.
The simplest Sobolev space used in this context corresponds to having a bound on the average
of the partial first ``weak derivatives'' of the density; other Sobolev spaces 
correspond to bounding additional derivatives.  A drawback of this approach
is that it does not apply to distributions whose densities have jump
discontinuities. Such jump discontinuities can arise in various applications, for example, when
objects under analysis must satisfy hard constraints.

\ignore{\rnote{Do we want to discuss Besov spaces at all?}  }
To address this, some authors have used the weaker assumption that the
density belongs to a Besov space
\citep{besov1959family,devore1993besov,masry1997multivariate,willett2007multiscale,acharya2017sample}.
In the simplest case, this allows jump discontinuities as long as the
function does not change very fast on average.  The precise definition,
which is quite technical (see \cite{devore1993besov}),
makes reference to the effect on a distribution of shifting the domain
by a small amount.

\medskip
\noindent {\bf The densities we consider.}
In this paper we analyze 
a clean and simple smoothness
assumption, which is a continuous analog of the notion of
shift-invariance that has recently been used for analyzing the
learnability of various types of discrete distributions
\citep{barbour1999poisson,daskalakis2013learning,DLS18asums}.
The assumption is based on the {\em shift-invariance of $f$ in direction $v$ at scale $\kappa$},
which, for a density $f$ over $\mathbb{R}^d$, a unit vector $v \in \mathbb{R}^d$, and a positive real value $\kappa$, we 
\ignore{\pnote{The discussion about full-rank covariance matrices seems to 
no longer be needed.}}
define to be
\[
\si(f,v,\kappa) \eqdef 
{\frac 1 \kappa} \cdot \sup_{\kappa' \in [0,\kappa]} \int_{\R^d} \left|f(x+ \kappa'  v) - f(x)\right| dx.\]
 We define the quantity $\si(f,\kappa)$ to be the worst case
of $\si(f,v,\kappa)$ over all directions $v$, i.e.\ 
\[
\si(f,\kappa) \eqdef \sup_{v: \Vert v \Vert_2 =1} \si(f,v,\kappa). 
\]
For any  
constant $c$,
we define the class of densities $\CSI(c, d)$  to consist of all $d$-dimensional densities $f$ with the property that $\si(f,\kappa) \le c$ for all $\kappa > 0.$ 

Our notion of shift-invariance provides a quantitative way of capturing the intuition that the density $f$ 
changes gradually on average in every direction.
Several natural classes fit nicely into this framework; for example, we note that $d$-dimensional standard normal distributions are easily shown to belong to $\CSI(1, d)$.  As another example, we will show later that any $d$-dimensional isotropic log-concave distribution belongs to $\CSI(O_d(1),d)$.


Many distributions arising in practice have light tails, and distributions with light
tails can in general be learned more efficiently.  To analyze learning shift-invariant distributions
in a manner that takes advantage of light tails when they are available, while accommodating heavier tails
when necessary, we define classes with different combinations of shift-invariant and tail behavior.
%
Given a nonincreasing function $g: \R^+ \to [0,1]$ which satisfies $\lim_{t \to +\infty} g(t) = 0$, we define the class of densities $\CSI(c,d,g)$ to consist of those $f \in \CSI(c,d)$ which have the additional property that for all $t > 0$, it holds that 
\[
\Prx_{\bx \leftarrow f}\left[|| \bx - \mu || > t \right] \leq g(t),
\]
where $\mu \in \R^d$ is the mean of the distribution $f.$  

As motivation for its study, we feel that  $\CSI(c,d,g)$ is a simple and easily understood class that exhibits an attractive tradeoff between expressiveness and tractability.  As we show, it is broad enough to include distributions of central interest such as multidimensional isotropic log-concave distributions, but it is also limited enough to admit efficient noise-tolerant density estimation algorithms.

\medskip

\noindent {\bf Our density estimation framework.}  We recall the standard notion of density estimation with respect to total variation distance.  Given a class $\calC$ of densities over $\mathbb{R}^d$, a density estimation algorithm for $\calC$ is given access to i.i.d. draws from $f$, where $f \in \calC$ is the unknown \emph{target density} to be learned.  For any $f \in \calC$, given any parameter $\eps > 0$, after making some number of draws depending on $d$ and $\eps$ the density estimation algorithm must output a description of a hypothesis density $h$ over $\mathbb{R}^d$ which, with high probability over the draws from $f$, satisfies $\dtv(f,h) \leq \eps$.  It is of interest both to bound the \emph{sample complexity} of such an algorithm (the number of draws from $f$ that it makes) and its running time.

Our learning results will hold even in a challenging model of \emph{noise-tolerant} density estimation for a class $\calC$.  In this framework, the density estimation algorithm is given access to i.i.d. draws from $f'$, which is a mixture $f' = (1-\eps)f + \eps f_{\mathrm{noise}}$ where $f \in \calC$ and $f_{\mathrm{noise}}$ may be any density. (We will sometimes say that such an $f'$ is an \emph{$\eps$-corrupted} version of $f$. 
This model of noise is sometimes referred to as \emph{Huber's contamination model}~\citep{huber1967behavior}.)  Now the goal of the density estimation algorithm is to output a description of a hypothesis density $h$ over $\mathbb{R}^d$ which, with probability at least (say) 9/10 over the draws from $f'$, satisfies $\dtv(f',h) \leq O(\eps)$. This is a challenging variant of the usual density estimation framework, especially for multidimensional density estimation.  In particular, there are simple distribution learning problems (such as learning a single Gaussian or product distribution over $\{0,1\}^n$) which are essentially trivial in the noise-free setting, but for which computationally efficient noise-tolerant learning algorithms have proved to be a significant challenge \cite{DKKLMNS16,DKKLMS18, Steinhardt18}.

\subsection{Results}

Our main positive result is a general algorithm which efficiently learns any class $\CSI(c,d,g)$ in the noise-tolerant model described above.  Given 
a constant $c$ and a tail bound 
$g$, 
we show that any distribution in the class $\CSI(c,d,g)$ can be noise-tolerantly learned to any error $O(\eps)$ with a sample complexity that depends on $c, g, \eps$ and $d$.   The running time of our algorithm is roughly 
quadratic in the sample complexity, and the sample complexity is 
$O_{c,d,g}(1) \cdot {\left(\frac {1}{\eps}\right)^{d+2}}$
 (see Theorem~\ref{thm:main-semiagnostic} in Section~\ref{sec:semi-agnostic} for a precise statement of the exact bound).
These bounds on the number of examples and running time do not depend on which member of $\CSI(c,d,g)$ is being learned.

%


\ignore{
\rnote{We need to work this out}To illustrate the power of our analysis, we note
that our positive results straightfowardly imply that the class of $d$-dimensional log-concave
distributions can be learned 
with the time and sample complexity described above,
solving an open problem posed by Diakonikolas \emph{et al.} \citep{diakonikolas2016learning}.
}

\medskip
\noindent {\bf Application:  Learning multivariate log-concave densities.} A multivariate density function $f$ over $\mathbb{R}^d$ is said to be \emph{log-concave} if there is an upper semi-continuous concave function $\phi: \mathbb{R}^d \to [-\infty,\infty)$ such that $f(x) = e^{\phi(x)}$ for all $x$.  Log-concave distributions arise in a range of contexts and have been well studied; see \citep{CDSS13,CDSS14,acharya2017sample,AcharyaDK15,CanonneDGR16,DKS16a} for work on density estimation of univariate (discrete and continuous) log-concave distributions.  In the multivariate case, 
\citep{KimSamworth14} 
gave a sample complexity lower bound (for squared Hellinger distance) which implies that $\Omega(1/\eps^{(d+1)/2})$ samples are needed to learn $d$-dimensional log-concave densities to error $\eps$.  More recently, \citep{diakonikolas2016learning} established the first finite sample complexity upper bound for multivariate log-concave densities, by giving an algorithm that \emph{semi-agnostically} (i.e. noise-tolerantly in a very strong sense) learns any $d$-dimensional log-concave density using $\tilde{O}_d(1/\eps^{(d+5)/2})$ samples.  The 
algorithm of \citep{diakonikolas2016learning} 
is not computationally efficient, and indeed, Diakonikolas et al.\ ask if there is an algorithm with running time polynomial in the sample complexity, referring to this as ``a challenging and important open question.'' A subsequent (and recent) work of 
Carpenter et al.~\citep{carpenter2018} showed that the maximum likelihood estimator (MLE) is statistically efficient (i.e., achieves near optimal sample complexity). However, we note that the MLE is computationally inefficient and thus has no bearing on the question of finding an efficient algorithm for learning log-concave densities.

We show that multivariate log-concave densities can be learned in
polynomial time as a special case of our main algorithmic result.  We
establish that any $d$-dimensional log-concave density is
$O_d(1)$-shift-invariant.  Together with well-known tail bounds on
$d$-dimensional log-concave densities, this easily yields that any
$d$-dimensional log-concave density belongs to $\CSI(c,d,g)$ where the
tail bound function $g$ is inverse exponential.
Theorem~\ref{thm:main-semiagnostic} then immediately implies the
following, answering the open question of
\citep{diakonikolas2016learning}:

\begin{theorem} \label{thm:log-concave-intro}
There is an algorithm with the following property:  Let $f$ be a unknown log-concave density over $\mathbb{R}^d$ and let $f'$ be an $\eps$-corruption of $f$.\ignore{ whose covariance matrix $\Sigma$ has full rank} Given any error parameter $\eps>0$ and confidence parameter $\delta > 0$ and access to independent draws from $f'$, the algorithm with probability $1-\delta$ outputs a hypothesis density ${h}:\mathbb{R}^d \rightarrow \mathbb{R}^{\geq 0}$ such that $\int_{x \in \mathbb{R}^d} |f'(x) - {h}(x)| \le O(\epsilon)$.  The algorithm runs in time $\tilde{O}_d(1/\eps^{2d+2}) \cdot \log^2(1/\delta)$ and uses $\tilde{O}_d(1/\eps^{d+2}) \cdot \log^2(1/\delta)$ many samples.  
\end{theorem}

While our sample complexity is quadratically larger than the optimal sample complexity for learning log-concave distributions (from \citep{diakonikolas2016learning}), such \emph{computational-statistical} tradeoffs are in fact quite common (see, for example, the work of~\cite{bhaskara2015sparse} which gives a faster algorithm for learning Gaussian mixture models by using more samples).

\medskip
\noindent {\bf A lower bound.}  We also prove a simple lower bound,
showing that any algorithm that learns shift-invariant $d$-dimensional
densities with bounded support to error $\eps$ must use
$\Omega\left(1/\epsilon^d\right)$ examples.  These densities may be
thought of as satisfying the strongest possible rate of tail decay as
they have zero tail mass outside of a bounded region (corresponding to
$g(t)=0$ for $t$ larger than some absolute constant).  This lower
bound shows that a sample complexity of at least $1/\eps^d$ is
necessary even for very structured special cases of our multivariate
density estimation problem.

\subsection{Our approach} \label{sec:approach}

For simplicity, and because it is a key component of our general algorithm, we first describe how our algorithm learns an $\eps$-error hypothesis when the target distribution belongs to $\CSI(c,d)$ and 
also has \emph{bounded support}:  all its mass is on points in the origin-centered ball of radius $1/2$.  

In this special case, analyzed in Section~\ref{sec:restricted}, our algorithm has two conceptual stages.  First, we smooth the density that we are to learn through convolution -- this is done in a simple way by randomly perturbing each draw.  This convolution
uses a kernel that damps the contributions to the density coming from 
high-frequency functions in its Fourier decomposition; intuitively, the shift-invariance of the target density ensures
that the convolved density (which is an average over small shifts of the original density) is close to the original density.
In the second conceptual stage,
the algorithm approximates relatively few Fourier coefficients of the smoothed density.  We show that an
inverse Fourier transformation using this approximation still provides an accurate approximation to
the target density.\footnote{We note that a simpler version of this approach, which only uses a smoothing kernel and does not employ Fourier analysis, can be shown to give a 
similar, but quantitatively worse,
results, such as a sample complexity of essentially $1/\eps^{2d}$ when $g(t)$ is zero outside of a bounded region.  However, this is worse than the lower bound of $\Omega(1/\eps^{d})$ by a quadratic factor, whereas our algorithm essentially achieves this optimal sample complexity.}

Next, in Section~\ref{sec:no-noise}, we consider the more general case in which the target distribution belongs to the class $\CSI(c,d,g)$ (so at this point we are not yet in the noise-tolerant framework).  Here the high-level idea of our approach is very straightforward:  it is essentially to reduce to the simpler special case (of bounded support and good shift-invariance in every direction) described above.  
(A crucial aspect of this transformation algorithm is that it uses only a small number of draws from the original shift-invariant distribution; we return to this point below.)  We can then use the algorithm for the special case to obtain a high-accuracy hypothesis, and perform the inverse transformation to obtain a high-accuracy hypothesis for the original general distribution.  We remark that while the conceptual idea is thus very straightforward, there are a number of technical challenges that must be met to implement this approach.  One of these is that it is necessary to truncate the tails of the original distribution so that an affine transformation of it will have bounded support, and doing this changes the shift-invariance of the original distribution.  Another is that the transformation procedure only succeeds with non-negligible probability, so we must run this overall approach multiple times and perform hypothesis selection to actually end up with a single high-accuracy hypothesis.  

In Section~\ref{sec:semi-agnostic}, we consider the most general case of noise-tolerant density estimation for $\CSI(c,d,g).$  Recall that in this setting the target density $f'$ is some distribution which need not actually belong to $\CSI(c,d,g)$ but satisfies $\dtv(f',f) \leq \eps$ for some density $f \in \CSI(c,d,g)$.  It turns out that this case can be handled using essentially the same algorithm as the previous paragraph.  We show that even in the noise-tolerant setting, our transformation algorithm will still successfully find a transformation as above that would succeed if the target density were $f \in \CSI(c,d,g)$ rather than $f'$.  (This robustness of the transformation algorithm crucially relies on the fact that it only uses a small number of draws from the given distribution to be learned.)  We then show that after transforming $f'$ in this way, the original algorithm for the special case can in fact learn the transformed version of $f'$ to high accuracy; then, as in the previous paragraph, performing the inverse transformation gives a high-accuracy hypothesis for $f'$.

In Section~\ref{sec:logconcave} we apply the above results to establish efficient noise-tolerant learnability of log-concave densities over $\R^d$. To apply our results, we need to have (i) bounds on the rate of tail decay, and (ii) shift-invariance bounds.  As noted earlier, exponential tail bounds on $d$-dimensional log-concave densities are well known, so it  remains to establish shift-invariance.  Using basic properties of log-concave densities, in Section~\ref{sec:logconcave} we show that any $d$-dimensional isotropic
log-concave density is $O_d(1)$-shift-invariant.  Armed with this bound, by applying our noise-tolerant learning result (Theorem~\ref{thm:main-semiagnostic}) we get that any $d$-dimensional isotropic log-concave density can be noise-tolerantly learned in time $\tilde{O}_d(1/\eps^{2d+2})$, using $\tilde{O}_d(1/\eps^{d+2})$ samples.
Log-concave distributions are shift-invariant even if they are only approximately isotropic.
We show that general log-concave distributions may be learned by bringing them into approximately
isotropic position with a preprocessing step, borrowing techniques from \cite{LovaszVempala07}.

\medskip

\noindent
{\bf The lower bound.}  As is standard, our lower bound (proved in Section~\ref{sec:lowerbound})  is obtained
via Fano's inequality.  We identify a large set $\cF$ of
bounded-support shift-invariant $d$-dimensional densities with the
following two properties: all pairs of densities from $\cF$ have
KL-divergence that is not too big (so that they are hard to tell
apart), but also have total variation distance that is not too small
(so that a successful learning algorithm is required to tell them
apart). The members of $\cF$ are obtained by choosing functions that
take one of two values in each cell of a $d$-dimensional checkerboard.
The two possible values are within a small constant factor of each
other, which keeps the KL divergence small.  To make the total
variation distance large, we choose the values using an
error-correcting code -- this means that distinct members of $\cF$
have different values on a constant fraction of the cells, which leads
to large variation distance.\ignore{ we also choose values are
  different enough to achieve the desired lower bound on the
  total-variation distance.}

\subsection{Related work} The most closely related work that we are aware
of was mentioned above: 
\citep{holmstrom1992asymptotic}
obtained bounds similar to ours for using kernel methods to learn densities that belong to various 
Sobolev spaces.  As mentioned above, these results do not directly apply for learning densities in
$\CSI(c,d,g)$ because of the possibility of jump discontinuities.\ignore{  \rnote{Removed the following parenthetical:  \green{(On
the other hand, it may be possible to obtain results similar to ours by arguing that, 
via a suitable convolution, a member of $\CSI(c,d,g)$ may be
approximated by a member of a Sobolev space.  The Fourier analysis
used in our approach appears to be a simpler and more illuminating path to our main result.)}}}
\citep{holmstrom1992asymptotic} also
proved a lower bound on the sample complexity of  algorithms that compute kernel
density estimates.  In contrast our lower bound holds for any density estimation algorithm,
kernel-based or otherwise.

The assumption that the target density belongs to a Besov
space 
(see \citep{kle2009smoothing})
makes reference to the effect of shifts on the distribution, as does shift-invariance.
We do not see any obvious containments between classes of functions defined through
shift-invariance and Besov spaces, but this is a potential topic for further research.  
%

Another difference with prior work is the ability of our approach to succeed in the challenging noise-tolerant learning model.  We are not aware of analyses for density estimation of densities belonging to Sobolev or Besov spaces that extend to the noise-tolerant setting in which the target density is only assumed to be close to some density in the relevant class.

As mentioned above, shift-invariance was used in the analysis of algorithms for
learning discrete probability distributions in \citep{barbour1999poisson,daskalakis2013learning}.
Likewise, both the discrete and continuous Fourier transforms have been used in the past to learn discrete probability distributions~\citep{diakonikolas2016optimal,diakonikolas2016fourier, DDKT16}.


\section{Preliminaries} \label{sec:prelims}

We write $B(r)$ to denote the radius-$r$ ball in $\mathbb{R}^d$, i.e. $B(r) = \{x \in \mathbb{R}^d: x_1^2 + \cdots + x_d^2 \leq r^2\}$.
If $f$ is a probability density over $\R^d$ and $S \subset \R^d$ is a subset of its domain, we write $f_S$ to denote the density of $f$ conditioned on $S$.

\subsection{Shift-invariance}
Roughly speaking, the shift-invariance of a distribution measures how
much it changes (in total variation distance) when it is subjected to
a small translation.  The notion of shift-invariance has typically
been used for discrete distributions (especially in the context of
proving discrete limit theorems, see e.g.~\citep{CGS11} and many
references therein).  We give a natural continuous analogue of this
notion below.

\begin{definition}~\label{def:shift-invariance} 
Given a probability density $f$ over $\mathbb{R}^d$, a unit vector $v$, and a positive real value $\kappa$, we say that the \emph{shift-invariance of $f$ in direction $v$ at scale $\kappa$}, denoted $\si(f,v,\kappa)$, is
\begin{equation} \label{eq:yyy}
\si(f,v,\kappa) \eqdef 
{\frac 1 \kappa} \cdot \sup_{\kappa' \in [0,\kappa]} \int_{\R^d} 
   \left|f(x+ \kappa' v) - f(x)\right| dx.
\end{equation}
\end{definition}

Intuitively, if $\si(f,v,\kappa)=\beta$, then for any direction (unit vector) $v$ 
the variation distance between $f$ and a shift of $f$ by $\kappa'$ in direction $v$ is at most $\kappa \beta$ for all $0 \leq \kappa' \leq \kappa$.
\blue{The factor ${\frac 1 \kappa}$ in the definition
means that $\si(f,v,\kappa)$ does not necessarily go to zero
as $\kappa$ gets small; the effect of shifting by $\kappa$ is measured
relative to $\kappa$.}

Let
\[
\si(f,\kappa) \eqdef \sup \{ \si(f,v,\kappa) : v \in \mathbb{R}^d, \Vert v \Vert_2 =1\}.
\]
For any 
constant $c$
we define the class of densities $\CSI(c, d)$  to consist of all $d$-dimensional densities $f$ with the property that $\si(f,\kappa) \le c$ for all $\kappa > 0.$  

\blue{We could obtain an equivalent definition if we 
removed the factor ${\frac 1 \kappa}$
from  the definition of $\si(f,v,\kappa)$, and required that 
$\si(f,v,\kappa) \leq c \kappa$ for all $\kappa > 0$.  This could of course
be generalized to enforce bounds on the modified $\si(f,v,\kappa)$ that are not linear in
$\kappa$.  We have chosen to focus on linear bounds in this paper to have cleaner theorems and proofs.}

\blue{We include ``sup'' in the definition due to the
fact that smaller shifts can sometimes have bigger effects.
For example, a sinusoid with period $\xi$ is unaffected by a
shift of size $\xi$, but profoundly affected by a
shift of size $\xi/2$.  Because of possibilities like this,
to capture the intuitive notion that ``small shifts do not
lead to large changes'', we seem to need to evaluate the
worst case over shifts of at most a certain size.}

\ignore{
\begin{remark}
For the rest of the paper, for the sake of simplicity we assume that the covariance matrix of $f \in \CSI(c,d)$ has full rank.  This assumption is without loss of generality because, as is implicit in the arguments of Lemma~\ref{lem:transformation}, if the rank is less than $d$ (and hence the density $f$ is supported in an affine subspace of dimension strictly less than $d$), then as a consequence of the shift-invariance of $f$, a small number of draws from $f$ will reveal the affine span of the entire distribution, and the entire algorithm can be carried out within that lower dimensional subspace.
\end{remark}
}

As described earlier, given a nonincreasing ``tail bound'' function $g: \R^+ \to (0,1)$ which is absolutely continuous and satisfies $\lim_{t \to +\infty} g(t) = 0$, we further define the class of densities $\CSI(c,d,g)$ to consist of those $f \in \CSI(c,d)$ which have the additional property that $f$ has \emph{$g$-light tails}, meaning that 
for all $t > 0$, it holds that 
$
\Prx_{\bx \leftarrow f}\left[|| \bx - \mu || > t \right] \leq g(t),$
where $\mu \in \R^d$ is the mean of $f.$

\begin{remark} \label{remark:tail-weight}
It will be convenient in our analysis to consider only tail bound functions $g$ that satisfy $\min\{r \in \R: g(r) \leq 1/2\} \geq 1/10$ (the constants $1/2$ and $1/10$ are arbitrary here and could be replaced by any other absolute 
positive
constants).  This is without loss of generality, since any tail bound function $g$ which does not meet this criterion can simply be replaced by a weaker tail bound function $g^\ast$ which does meet this criterion, and clearly if $f$ has $g$-light tails then $f$ also has $g^\ast$-light tails.
\end{remark}

We will (ab)use the notation $g^{-1}(\epsilon)$ to mean
$\inf \{ t : g(t) \leq \epsilon \}$.

The complexity of learning with a tail bound $g$ will be expressed in part
using
\[
I_g \eqdef \int_0^{\infty} g(\sqrt{z}) \; dz.
\]
We remark that the quantity $I_g$ is the ``right" quantity in the sense that the integral $I_g$ is finite as long as the density has ``non-trivial decay". More precisely, note that by Chebyshev's inequality, $g(\sqrt{z}) = O(z^{-1})$. Since the integral $\int O(z^{-1}) dz$ diverges, this means that if $I_g$ is finite, then the density $f$ has a decay sharper than the trivial decay implied by Chebyshev's inequality.

\subsection{Fourier transform of high-dimensional distributions}

%

In this subsection we gather some helpful facts from multidimensional
Fourier analysis.



While it is possible 
to do Fourier analysis over $\mathbb{R}^d$, 
in this paper, we will only do Fourier analysis for functions $f \in  L_1([-1,1]^d)$.  
\begin{definition}
For any function $f \in L_1([-1,1]^d)$,
we define
$\widehat{f}: \mathbb{R}^d \rightarrow \mathbb{C}$
by
$
\widehat{f}(\xi) = \int_{x \in \mathbb{R}^d} f(x) \cdot e^{ \pi i \cdot \langle \xi, x \rangle} dx.$
\end{definition}

Next, we recall the following standard claims about Fourier transforms of functions, which may be found,
for example, in \citep{smith1995handbook}.
\ignore{\pnote{I did not find the following in {\em Fourier Analysis}, by K\"orner.  
Some of them are in \citep{smith1995handbook}, but, honestly, I'm only guessing that the others are there.  Do either of you have
a reference handy that we could use? \red{Rocco:  Are we happy with this now?}}
}
\begin{claim}~\label{clm:convolution}
For $f,g \in L_1([-1,1]^d)$  let   
$
{h}(x) = \int_{y \in \mathbb{R}^d} f(y) \cdot g(x-y) dy
$ denote the convolution $h = f \ast g$ of $f$ and $g$.
Then for any $\xi \in \mathbb{R}^n$, we have $\widehat{h}(\xi) = \widehat{f}(\xi) \cdot \widehat{g}(\xi)$. 
\end{claim}

Next, we recall Parseval's identity on the cube.

\begin{claim}[Parseval's identity]~\label{clm:Parseval}
For $f: [-1,1]^d \rightarrow \mathbb{R}$ such that $f \in L_2( [-1,1]^d)$, it holds that
$
 \int_{ [-1,1]^d} f(x)^2 dx = \frac{1}{2^d} \cdot \sum_{\xi  \in \mathbb{Z}^d} |\widehat{f}(\xi)|^2. 
$
\end{claim}
The next claim says that the Fourier inversion formula can be applied to any sequence in $\ell_2(\mathbb{Z}^d)$ to obtain a function whose Fourier series is identical to the given sequence. 
\begin{claim}[Fourier inversion formula] \label{clm:inversion}\
For any $g: \mathbb{Z}^d \rightarrow \mathbb{C}$ such that $\sum_{\xi \in \mathbb{Z}^d} |{g}(\xi)^2| <\infty$, the function
$
h(x) = \sum_{\xi \in \mathbb{Z}^d} \frac{1}{2^d} \cdot g(\xi) \cdot e^{ \pi i \cdot \langle \xi, x \rangle},
$
is well defined and satisfies $\widehat{h}(\xi) = g(\xi)$ for all $\xi \in \mathbb{Z}^d$. \end{claim}
We will also use Young's inequality: 
\begin{claim}[Young's inequality] \label{clm:Young}
Let $f \in L_p([-1,1]^d)$, $g \in L_q([-1,1]^d)$, $1 \le p,q ,r \le \infty$, 
such that $1 + 1/r = 1/p + 1/q$. Then $\Vert f \ast g \Vert_r \le \Vert f \Vert_p \cdot \Vert g \Vert_q$. 
\end{claim}

\subsection{A useful mollifier} \label{sec:mollifier}

Our algorithm and its analysis require the existence of a compactly supported distribution with fast decaying Fourier transform. Since the precise rate of decay is not very important, 
we use the $C^{\infty}$ function $b: [-1,1] \rightarrow \mathbb{R}^+$ as follows: 
\begin{eqnarray}~\label{eq:b}
b(x) = 
\begin{cases}
c_0 \cdot e^{-\frac{x^2}{1-x^2}} &\textrm{if } |x|<1 \\
0 &\textrm{if } |x|=1.
\end{cases}
\end{eqnarray}
Here $c_0 \approx 1.067$ is chosen so that 
$b$ is a pdf;  by symmetry, its mean is $0$.
(This function has previously been used as a
mollifier \citep{kane2010exact,diakonikolas2010bounded}.)
The following fact can be found in \cite{johnson2015saddle} (while it is proved only for $\xi \in \mathbb{Z}$, it is easy to see that the same proof holds if $\xi \in \mathbb{R}$). 
\ignore{\pnote{How is the following fact proved? \red{Rocco:  I think I missed part of the discussion on this on the recent phone call --- were we going to cite a paper of Jelani's or something?}}}
\begin{fact}~\label{eq:FTB}
For $b:[-1,1] \rightarrow \mathbb{R}^+$ defined in (\ref{eq:b}) and $\xi \in \mathbb{Z} \setminus \{0\}$,  we have that $|\widehat{b}(\xi)| \le e^{-\sqrt{|\xi|}} \cdot |\xi|^{-3/4}$. 
\end{fact}

Let us now define the function $b_{d,\gamma}: \mathbb{R}^d \rightarrow \mathbb{R}^+$ as 
$b_{d,\gamma}(x_1, \ldots, x_d) = \frac{1}{\gamma^d} \cdot \prod_{j=1}^d b(x_j/\gamma).$
Combining this definition and Fact~\ref{eq:FTB}, we have the following claim:\ignore{\rnote{Sorry to be dense but I don't see exactly how we get this from the definition of $b_{d,\gamma}$ and Fact~\ref{eq:FTB}}
\anote{Rocco, I think the idea is that since $b(\cdot)$ integrates to $1$, its Fourier transform at any point is bounded by $1$. $b_{d,\gamma}$ is just the product distribution where each coordinate is $b$ (after a translation by $\gamma$).}}
\begin{claim}\label{clm:Fourier-b}
For $\xi \in \mathbb{Z}^d$ with $\Vert \xi \Vert_\infty \ge t$, we have $|\widehat{b_{d,\gamma}}(\xi)| \le e^{-\sqrt{\gamma \cdot t}} \cdot (\gamma \cdot t)^{-3/4}$. 
\end{claim}
The next fact is immediate from (\ref{eq:b}) and the definition of $b_{d,\gamma}$:
\begin{fact}~\label{fact:b-sup-density}
$\Vert b_{d,\gamma} \Vert_\infty = (c_0/\gamma)^d$ and as a consequence,
 $\Vert b_{d,\gamma} \Vert_2^2 \le  (c_0/\gamma)^{2d}$.
\end{fact}


\section{A restricted problem:  learning shift-invariant distributions with bounded support} 
\label{sec:restricted}

As sketched in Section~\ref{sec:approach}, we begin by presenting and
analyzing a density estimation algorithm for densities that, in addition
to being shift-invariant, have support bounded in $B(1/2)$.  
Our analysis also captures the fact that,
to achieve accuracy $\epsilon$, an algorithm often only
needs the density to be learned to have shift invariance at
a scale slightly finer than $\epsilon$.  

\begin{lemma} \label{lem:finite-support}
There is an algorithm \textsf{learn-bounded} with the following property:  
For all constant $d$, 
for all $\epsilon, \delta > 0$, all $0 < \kappa < \epsilon < 1/2$,
and all $d$-dimensional densities $f$ 
with support in $B(1/2)$
such that
$\kappa \si(f, \kappa) \leq \epsilon/2$,
given 
access to independent draws from $f$, the algorithm runs in time 
 \[
O_{d}\bigg( \frac{1}{\epsilon^2}
          \left( \frac{1}{\kappa} \right)^{2 d}
          \log^{4 d} \left( \frac{1}{\kappa} \right)
          \log \left( \frac{1}{\kappa \delta} \right) \bigg)
\]
uses 
\[
O_{d}\bigg( \frac{1}{\epsilon^2}
          \left( \frac{1}{\kappa} \right)^d
          \log^{2d} \left( \frac{1}{\kappa} \right)
          \log \left( \frac{1}{\kappa \delta} \right)\bigg) 
\]
samples, and with probability $1-\delta$, outputs a hypothesis
 $h:[-1,1]^d \rightarrow \mathbb{R}^+$ such that $\int_{x \in \mathbb{R}^d} |f(x) - h(x)| \le \epsilon$. 
 
Further, 
given any point $z \in [-1,1]^d$, $h(z)$ can be computed in time 
$
O_{d} \left( {\frac {\log^{2d}
      (1/\kappa)}{\kappa^d}} \right)$
and satisfies
$h(z) \le 
O_{d} \left( {\frac {\log^{2d} (1/\kappa)}{\kappa^d}} \right)$.
\end{lemma}

\begin{proof}
Let $0 <  \gamma := \frac{\kappa}{\sqrt{d}}$, 
and let us define $q= f \ \ast  \ b_{d,\gamma}$.  (Here $\ast$ denotes convolution and $b_{d,\gamma}$ is the mollifier defined in Section~\ref{sec:mollifier}.)  We make a few simple observations about $q$:
\begin{itemize}
\item[(i)] Since $\gamma \leq 1/2,$  we have that $q$ is a density supported on $B(1)$. 
\item[(ii)] Since $d$ is a constant, 
a draw from $b_{d,\gamma}$  can be generated in 
constant time.   
Thus given a draw from $f$, 
one can generate a draw from $q$ 
in constant time, simply by generating a draw from $b_{d,\gamma}$ and adding it to the draw from $f$.
\item[(iii)]  By Young's inequality (Claim~\ref{clm:Young}), we have that 
$\Vert q \Vert_2 \le \Vert f \Vert_1 \cdot \Vert b_{d,\gamma} \Vert_2$. Noting that $f$ is a density and thus $\Vert f \Vert_1=1$ and applying Fact~\ref{fact:b-sup-density}, we obtain that $\Vert q \Vert_2$ is finite. As a consequence, the Fourier coefficients of $q$ are well-defined. 
\end{itemize}
\noindent \emph{Preliminary analysis.}
We first observe that because $b_{d,\gamma}$ is supported on $[-\gamma,\gamma]^d$, the distribution $q$ may be viewed as an average of different shifts of $f$ where each shift is by a distance at most $\gamma \sqrt{d} \leq \kappa$. Fix any direction $v$ and consider a shift of $f$ in direction $v$ by some distance at most $\gamma \sqrt{d} \leq \kappa$. Since 
$\kappa \si(f, \kappa) \leq \epsilon/2$,
we have that the variation distance between $f$ and this shift in direction $v$ is at most $\epsilon/2$.  Averaging over all such shifts, it follows that
\begin{equation}~\label{eq:tv-distance}
\dtv(q,f) \leq \eps/2.
\end{equation}

Next, we observe that by Claim~\ref{clm:convolution}, for any $\xi \in
\mathbb{Z}^d$, we have $\widehat{q}(\xi) = \widehat{f}(\xi) \cdot
\widehat{b_{d,\gamma}} (\xi)$. Since $f$ is a pdf, $|\widehat{f}(\xi)
| \le 1$, and thus we have $|\widehat{q}(\xi)| \leq
|\widehat{b_{d,\gamma}} (\xi)|$. Also, for any parameter $k \in
\mathbb{Z}^+$, define $C_k = \{\xi \in \mathbb{Z}^d: \Vert \xi
\Vert_\infty = k\}$.
Let us fix another parameter $T$ (to be determined later). 
Applying Claim~\ref{clm:Fourier-b}, we obtain
\begin{eqnarray}
&& \sum_{\xi: \Vert \xi \Vert_\infty > T} |\widehat{q}(\xi)|^2 \le  \sum_{\xi: \Vert \xi \Vert_\infty > T} |\widehat{b_{d,\gamma}}(\xi)|^2 \le \sum_{k > T} \sum_{\xi: \Vert \xi \Vert_\infty =k}|\widehat{b_{d,\gamma}}(\xi)|^2  \nonumber \\ 
&& \le \sum_{k > T} |C_k| \cdot    e^{-2 \cdot \sqrt{\gamma \cdot k}} \cdot (\gamma \cdot k)^{-3/2} \  
\leq \sum_{k > T} (2k+1)^d \cdot    e^{-2 \cdot \sqrt{\gamma \cdot k}} \cdot (\gamma \cdot k)^{-3/2}. \nonumber
\end{eqnarray}
An easy calculation shows that 
if $T \ge \frac{4d^2}{\gamma} \cdot \ln^2 \bigg(\frac{d}{\gamma}\bigg)$, then $\sum_{\xi: \Vert \xi \Vert_\infty > T} |\widehat{q}(\xi)|^2 \le  2 (2T+1)^d \cdot e^{-2 \cdot \sqrt{\gamma \cdot T}} \cdot  (\gamma \cdot T)^{-3/2}.$
If we now set $T$ to be 
$\frac{4d^2}{\gamma} \cdot \ln^2 \bigg(\frac{d}{\gamma}\bigg) + \frac{1}{\gamma} \cdot \ln^2 \bigg( \frac{8 }{\epsilon}\bigg)$, then $\sum_{\xi: \Vert \xi \Vert_\infty > T} |\widehat{q}(\xi)|^2  \le \frac{\epsilon^2}{8}.$ 

\emph{The algorithm.}
We first observe that for any $\xi \in \mathbb{Z}^d$, the Fourier coefficient $\widehat{q}(\xi)$ can be 
estimated
to good accuracy using relatively few draws from $q$ (and hence from $f$, recalling (ii) above).  More precisely, as an easy consequence of the definition of the Fourier transform, we have:
\begin{observation}~\label{obs:compute-Fourier}
For any $\xi \in \mathbb{Z}^d$, the Fourier coefficient $\widehat{q}(\xi)$ can be 
estimated
to 
within
additive error 
of magnitude at most $\eta$
with confidence $1-\beta$ using $O(1/\eta^2 \cdot \log(1/\beta))$ draws from $q$. 
\end{observation}
Let us define the set $\mathsf{Low}$ of low-degree Fourier coefficients as $\mathsf{Low} = \{\xi \in \mathbb{Z}^d : \Vert \xi \Vert_\infty \le T\}$. 
Thus, $|\mathsf{Low}| \le (2T+1)^d$.  Thus, using $S = O(\eta^{-2} \cdot \log (T/ \delta))$ draws from $f$, by Observation~\ref{obs:compute-Fourier}, with probability $1-\delta$, we can compute
a set of values $\{\widehat{u}(\xi)\}_{\xi \in \mathsf{Low}}$ such that 
\begin{equation}\label{eq:approx}
\textrm{For all }\xi \in \mathsf{Low}, \ |\widehat{u}(\xi) - \widehat{q}(\xi)| \le \eta. 
\end{equation}
Recalling (ii), 
the sequence $\{\widehat{u}(\xi)\}_{\xi \in \mathsf{Low}}$ 
can be computed in $O(|S| \cdot |\mathsf{Low}|)$ time. 
Define $\widehat{u}(\xi) =0$ for $\xi \in \mathbb{Z}^d \setminus \mathsf{Low}$. Combining (\ref{eq:approx}) with this, we get 
\begin{eqnarray}
\sum_{\xi \in \mathbb{Z}^d}  |\widehat{u}(\xi) - \widehat{q}(\xi)|^2 &\le& \sum_{\xi \in \mathsf{Low}} \ |\widehat{u}(\xi) - \widehat{q}(\xi)|^2 + \sum_{\xi \not \in \mathsf{Low}} \ |\widehat{u}(\xi) - \widehat{q}(\xi)|^2  \nonumber \\&\le& \sum_{\xi \in \mathsf{Low}} \ |\widehat{u}(\xi) - \widehat{q}(\xi)|^2 + \frac{\epsilon^2 }{8} \nonumber \\
&\le& | \mathsf{Low}| \cdot \eta^2 +  \frac{\epsilon^2}{8}  \le (2T+1)^d \cdot \eta^2 + \frac{\epsilon^2 }{8}. \nonumber 
\end{eqnarray}
Thus, setting $\eta$ as $\eta^2 = (2T+1)^{-d} \cdot \frac{\epsilon^2}{8} $, we get that \begin{eqnarray}\sum_{\xi \in \mathbb{Z}^d}  |\widehat{u}(\xi) - \widehat{q}(\xi)|^2 \le \frac{\epsilon^2}{4}. \label{eq:eta-bound} \end{eqnarray}
Note that by definition $\widehat{u} : \mathbb{Z}^d \rightarrow \mathbb{C}$ satisfies $\sum_{\xi \in \mathbb{Z}^d} |\widehat{u}(\xi) |^2 < \infty$. Thus, we can apply the Fourier inversion formula (Claim~\ref{clm:inversion}) to obtain a function $u: [-1,1]^d \rightarrow \mathbb{C}$ such that 
\begin{equation}~\label{eq:Parseval-eqn}
\int_{[-1,1]^d} |u(x) - q(x)|^2  dx = \frac{1}{2^d} \cdot \big( \sum_{\xi \in \mathbb{Z}^d} |\widehat{u}(\xi) - \widehat{q}(\xi)|^2 \big) \le \frac{\epsilon^2}{4 \cdot 2^d},
\end{equation}
where the first equality follows by Parseval's identity (Claim~\ref{clm:Parseval}). By the Cauchy-Schwarz inequality, 
\[
\int_{[-1,1]^d} |u(x) - q(x)|  dx  \leq \sqrt{2^d} \cdot \sqrt{\int_{[-1,1]^d} |u(x) - q(x)|^2  dx }.
\]
Plugging in (\ref{eq:Parseval-eqn}), we obtain 
$
\int_{[-1,1]^d} |u(x) - q(x)|  dx \le \frac{\epsilon}{2}.$
Let us finally define $h$ (our final hypothesis), $h: [-1,1]^d \rightarrow \mathbb{R}^+$, as follows: 
$h(x) = \max \{0, \mathsf{Re}(u(x))\}$. Note that since $q(x)$ is a non-negative real value for all $x$, we have 
\begin{equation} \label{eq:pickle}
\int_{[-1,1]^d} |h(x) - q(x)|  dx \le \int_{[-1,1]^d} |u(x) - q(x)|  dx \le \frac{\epsilon}{2}. 
\end{equation}
Finally, recalling that by (\ref{eq:tv-distance}) we have $\dtv(f,q) \le \frac{\epsilon}{2}$, it follows that 
$
\int_{[-1,1]^d} |h(x) - f(x)|  dx  \le \epsilon.$

\emph{Complexity analysis.}
We now analyze the time and sample complexity of this algorithm as well as the complexity of computing $h$. First of all, observe that plugging in the value of $\gamma$ and recalling that $d$ is a constant, we get that $T  = 
\frac{4d^2}{\gamma} \cdot \ln^2 \bigg(\frac{d}{\gamma}\bigg) + \frac{1}{\gamma} \cdot \ln^2 \bigg( \frac{8 }{\epsilon}\bigg)
 =
O\left({\frac {\log^2(1/\kappa)}{\kappa}}\right)$. Combining this with the choice of $\eta$ (set just above (\ref{eq:eta-bound})),  we get that the algorithm uses
\begin{align*}
& S = O\bigg( \frac{1}{\eta^2} \cdot \log \bigg(\frac{|\mathsf{Low}|}{\delta} \bigg) \bigg) 
= O\bigg( \frac{1}{\eta^2} \cdot \log \bigg(\frac{T}{\delta} \bigg) \bigg) 
= O\left( \frac{(2 T + 1)^d \cdot \log \bigg(\frac{T}{\delta} \bigg)}{\epsilon^2} \right)  \\
& = O_{d}\bigg( \frac{1}{\epsilon^2}
          \left( \frac{1}{\kappa} \right)^d
          \log^{2d} \left( \frac{1}{\kappa} \right)
          \log \left( \frac{1}{\kappa \delta} \right)\bigg) 
\end{align*}
draws from $p$.
Next, as we have noted before, computing the sequence $\{\widehat{u}(\xi)\}$ takes time 
\begin{align*}
O(S \cdot |\mathsf{Low}|)
  & = O_{d}\bigg( \frac{1}{\epsilon^2}
          \left( \frac{1}{\kappa} \right)^d
          \log^{2d} \left( \frac{1}{\kappa} \right)
          \log \left( \frac{1}{\kappa \delta} \right) T^d \bigg) 
     \\
  & = O_{d}\bigg( \frac{1}{\epsilon^2}
          \left( \frac{1}{\kappa} \right)^{2 d}
          \log^{4 d} \left( \frac{1}{\kappa} \right)
          \log \left( \frac{1}{\kappa \delta} \right) \bigg). \\
\end{align*}

To compute the function $u$ (and hence $h$) at any point $x \in
[-1,1]^d$ takes time
$O(|\mathsf{Low}|) = 
  O_{d} \left( {\frac {\log^{2d}
      (1/\kappa)}{\kappa^d}} \right).$ This is because
the Fourier inversion formula (Claim~\ref{clm:inversion}) has at most
$O(|\mathsf{Low}|)$ non-zero terms.

Finally, we prove the upper bound on $h$.  
If the training examples are $x_1,...,x_S$, then
for any $z \in [-1,1]^d$, we have
\begin{align*}
& h(z) \leq | u(z) |
     = \left| \sum_{\xi \in \mathsf{Low}} \frac{1}{2^d} \cdot \hat{u} (\xi) \cdot e^{ \pi i \cdot \langle \xi, z \rangle} \right|
     = \left| \sum_{\xi \in \mathsf{Low}} \frac{1}{2^d} \cdot 
     \left( \frac{1}{S} \sum_{t=1}^S e^{\pi i \langle \xi, x_t \rangle} 
       \right)
         \cdot e^{ \pi i \cdot \langle \xi, z \rangle} \right| \\
&      \leq \frac{|\mathsf{Low}|}{2^d} 
    =  O_{d} \left( {\frac {\log^{2d} (1/\kappa)}{\kappa^d}} \right),
\end{align*}
completing the proof.
\end{proof}

With an eye towards our ultimate goal of obtaining noise-tolerant density estimation algorithms, the next corollary says that the algorithm in Lemma~\ref{lem:finite-support} is robust to noise. All the parameters have the same meaning and relations as in Lemma~\ref{lem:finite-support}.  

\begin{corollary}~\label{corr:agnostic}
Let $f'$ be a density supported in $B(1/2)$\footnote{Looking ahead, while in general an  ``$\eps$-noisy'' version of $f$ need not be supported in $B(1/2)$, the reduction we employ will in fact ensure that we only need to deal with noisy distributions that are in fact supported in $B(1/2).$
} such that there is a $d$-dimensional density $f$ satisfying the following two properties: (i) $f$ satisfies all the conditions in the hypothesis of Lemma~\ref{lem:finite-support}, and (ii) $f'$ is an $\eps$-corrupted version of $f$, i.e. $f'=(1-\eps)f + \eps f_{\mathrm{noise}}$ for some density $f_{\mathrm{noise}}.$ 
Then given access to samples from $f'$,  the algorithm
 \textsf{learn-bounded}  returns a hypothesis $h: [-1,1]^d \rightarrow \mathbb{R}^+$ which satisfies $\int_{x \in \mathbb{R}^d} |f'(x) - h(x)| \le 2 \epsilon$. All the other guarantees including the sample complexity and time complexity remain the same as Lemma~\ref{lem:finite-support}. 
\end{corollary}

\begin{proof}
The proof of Lemma~\ref{lem:finite-support} can be broken down into
two parts: 
\begin{itemize}
\item $f$ can be approximated by $q$, and
\item $q$ can be learned.
\end{itemize}
The argument that $q$ can be learned only used two facts about it:
\begin{itemize}
\item it is supported in $[-1,1]^d$, and
\item it has few nonzero Fourier coefficients.
\end{itemize}
So, now 
consider the distribution $q' = f' \ast b_{d,\gamma}$ where $b_{d,\gamma}$ is the same distribution as in Lemma~\ref{lem:finite-support}.
Because $q'$ is the result of convolving $f'$ (a density supported in $B(1/2)$ with $b_{d,\gamma}$, it is supported in 
$[-1,1]^d$, and has the same Fourier concentration property that
we used for $q$.  Thus,
the algorithm will return a hypothesis distribution $h(x)$ such that the analogue of (\ref{eq:pickle}) holds, i.e.
\begin{equation}~\label{eq:h-guarantee} \int_{[-1,1]^d} |h(x)  - q'(x)| dx \le \frac{\epsilon}{2}.\end{equation} 
Recalling that the density $f'$ can be expressed as 
$(1-\epsilon) f + \epsilon f_{\mathrm{noise}}$ where $f_{\mathrm{noise}}$ is some density supported in $B(1/2)$, we now have
\begin{eqnarray*}
\dtv(q', f') = \dtv(f' \ast b_{d,\gamma} , f') &=& \dtv((1-\epsilon) f \ast b_{d,\gamma} + \epsilon f_{\mathrm{noise}} \ast b_{d,\gamma}, (1-\epsilon) f + \epsilon f_{\mathrm{noise}})  \\
&\le& (1-\epsilon) \dtv(f \ast b_{d,\gamma}, f) + \epsilon \dtv(f_{\mathrm{noise}} \ast b_{d,\gamma}, f_{\mathrm{noise}})\\
&\le& \epsilon/2 + \epsilon \le 3 \epsilon/2. 
\end{eqnarray*}
The penultimate inequality uses (\ref{eq:tv-distance}) and the fact that the total variation distance between any two distributions is bounded by $1$. Combining the above with (\ref{eq:h-guarantee}), the corollary is proved.
\end{proof}


\section{Density estimation for densities in $\CSI(c,d,g)$}
\label{sec:no-noise}

Fix any nonincreasing tail bound function $g: \R^+ \to [0,1]$ which satisfies $\lim_{t \to +\infty} g(t) = 0$
and the condition $\min\{r \in \R: g(r) \leq 1/2\}\geq 1/10\}$ of Remark~\ref{remark:tail-weight}
and any constant $c \geq 1$.
In this section we prove the following theorem which gives a density estimation algorithm for the class of distributions
$\CSI(c,d,g)$:
\begin{theorem} \label{thm:no-noise}
For any $c,g$ as above and any $d \geq 1$,  there is an algorithm with the following property:  Let $f$ be any target density 
(unknown to the algorithm) which belongs to $\CSI(c,d,g)$. 
Given any error parameter $0 < \eps < 1/2$ and confidence parameter $\delta > 0$ and access to independent draws from $f$, the algorithm with probability $1-O(\delta)$ outputs a hypothesis  $h:[-1,1]^d \rightarrow \mathbb{R}^{\geq 0}$ such that $\int_{x \in \mathbb{R}^d} |f(x) - h(x)| \le O(\epsilon)$. 

The algorithm runs in time 

\ignore{
O\left(
  \left(
   (1 + g^{-1}(\epsilon))^{2d}
          \left( \frac{1}{\epsilon} \right)^{2 d + 2}
          \log^{4 d} \left( \frac{1 + g^{-1}(\epsilon)}{\epsilon} \right)
          \log \left( \frac{1 + g^{-1}(\epsilon)}{\epsilon \delta} \right) 
  + I_g \right) \log \frac{1}{\delta}
\right)
}
\[
O_{c,d}\left(
  \left(
   (g^{-1}(\epsilon))^{2d}
          \left( \frac{1}{\epsilon} \right)^{2 d + 2}
          \log^{4 d} \left( \frac{g^{-1}(\epsilon)}{\epsilon} \right)
          \log \left( \frac{g^{-1}(\epsilon)}{\epsilon \delta} \right) 
  + I_g \right) \log \frac{1}{\delta}
\right)
\] 
and uses 
\ignore{
O\left(
  \left(
   (1 + g^{-1}(\epsilon))^{2d}
          \left( \frac{1}{\epsilon} \right)^{2 d + 2}
          \log^{4 d} \left( \frac{1 + g^{-1}(\epsilon)}{\epsilon} \right)
          \log \left( \frac{1 + g^{-1}(\epsilon)}{\epsilon \delta} \right) 
  + I_g \right) \log \frac{1}{\delta}
\right)
}

\[
O_{c,d}\left(
 \left(
 (g^{-1}(\epsilon))^{d}
          \left( \frac{1}{\epsilon} \right)^{d+2}
          \log^{2d} \left( \frac{g^{-1}(\epsilon)}{\epsilon} \right)
          \log \left( \frac{g^{-1}(\epsilon)}{\epsilon \delta} \right)
 + I_g \right) \log \frac{1}{\delta}
 \right)
\] samples.  
\end{theorem}

\subsection{Outline of the proof}
Theorem~\ref{thm:no-noise} is proved by a reduction to Lemma~\ref{lem:finite-support}. The main ingredient in the proof of Theorem~\ref{thm:no-noise} is a ``transformation algorithm'' with the following property:  given as input access to i.i.d.\ draws from any density $f  \in\CSI(c,d,g)$, the algorithm constructs parameters which enable draws from the density $f$ to be transformed into draws from another density, which we denote $r$.  The density $r$ is obtained by
approximating $f$ after conditioning on a non-tail sample, and scaling the result so that it lies in
a ball of radius $1/2$.

Given such a transformation algorithm, the approach to learn $f$ is
clear: we first run the transformation algorithm to get access to
draws from the transformed distribution $r$.  We then use draws from
$r$ to run the algorithm of Lemma~\ref{lem:finite-support} to learn
$r$ to high accuracy.  (Intuitively, the
error relative to $f$ of the final hypothesis density is $O(\eps)$
because at most $O(\eps)$ comes from the conditioning and at most
$O(\eps)$ from the algorithm of Lemma~\ref{lem:finite-support}.) We
note that while this high-level approach is conceptually
straightforward, a number of technical complications arise; for
example, our transformation algorithm only succeeds with some
non-negligible probability, so we must run the above-described
combined procedure multiple times and perform hypothesis testing to
identify a successful final hypothesis from the resulting pool of
candidates.

The rest of this section is organized as follows: In
Section~\ref{sec:setup} we give various necessary technical
ingredients for our transformation algorithm.  We state and prove the
key results about the transformation algorithm in
Section~\ref{sec:transformation}, and we use the transformation
algorithm to prove Theorem~\ref{thm:no-noise} in
Section~\ref{sec:proof-of-no-noise}.

\subsection{Technical ingredients for the transformation algorithm} \label{sec:setup}


As sketched earlier, our approach will work with a density 
obtained by conditioning $f \in \si(c,d)$ on lying in a certain ball that has mass close to 1 under $f$.\ignore{(intuitively, this corresponds to the condition used in the learning
algorithm of Section~3; recall that the densities learned by that
algorithm were assumed to lie in $B(1/2)$).}
While we know that the original density $f \in \si(c,d)$ has
good shift-invariance, we will further need the
conditioned distribution to also have good shift-invariance
in order for the \textsf{learn-bounded} algorithm of Section~\ref{sec:restricted} to work.  Thus we require the following
simple lemma, which shows that conditioning a density $f \in \si(c,d)$
on a region of large probability cannot hurt its shift invariance too much.

\begin{lemma}~\label{lem:restrict-si}
Let $f \in \si(c,d)$ and 
let $B$ be a ball such that $\Pr_{\bx \sim f} [\bx \in B] \ge 1- \delta$ where $\delta < 1/2$. 
If $f_B$ is the density of $f$ conditioned on $B$,
then,
for all $\kappa > 0$,
$\si(f_B,\kappa) \leq \frac{4\delta}{\kappa} + 2 c$.
\end{lemma}
\begin{proof}
Let $v$ be any unit vector in $\mathbb{R}^d$.\ignore{\footnote{spurious line?: Let $f_v$  (resp. $f_{B,v}$) be the densities obtained by projecting $f$ (resp. $f_B$) to direction $v$. }}
Note that $f$ can be expressed as $(1-\delta)f_B$ + $\delta \cdot f_{err}$ where $f_{err}$ is some other density. As a consequence, for any $\kappa > 0$,
using the triangle inequality we have that
\begin{align*}
\int_x |f(x) - f(x + \kappa v )| dx &\ge (1-\delta) \int_x |f_B(x) - f_B(x + \kappa   v)| dx \\
& \hspace{0.5in} - \delta  \int_x |f_{err}(x) - f_{err}(x + \kappa   v)| dx.
\end{align*}
Since $f \in \calC_\si(c,d)$ the left hand side is at most  $c \kappa$, whereas the subtrahend  on the right hand side is trivially at most $2\delta$. Thus, we get
\begin{equation}~\label{eq:shift-B-bound-1}
\int_x |f_B(x) - f_B(x + \kappa  v)| dx  \leq \frac{2\delta}{1-\delta} + \frac{c \kappa}{1-\delta},
\end{equation}
completing the proof.
\end{proof}


If $f$ is an unknown target density then of course its mean 
is also unknown, and thus we will need to approximate it
using draws from $f$. 
To do this, it will be helpful to convert our condition on the tails
of $f$ to bound the variance of $|| \bx - \mu ||$, where $\bx \sim f.$
\begin{lemma}
\label{l:var}
For any $f \in \CSI(c, d, g)$, we have
$
\E_{\bx \sim f} [ || \bx - \mu ||^2 ] \leq  I_g.
$
\end{lemma}
\begin{proof}
We have
$
\E_{\bx \sim f} [ || \bx - \mu ||^2 ]
 = \int_0^{\infty} \Pr_{\bx \sim f} [ || \bx - \mu ||^2 \geq z ] \; dz
 \leq \int_0^{\infty} g(\sqrt{z}) \; dz
 = I_g.
$
\end{proof}

The following easy proposition gives a guarantee on the quality of the empirical mean: 
\begin{lemma}~\label{lem:estimate-mean}
For any $f \in \CSI(c,d,g)$,
if $\mu \in \R^d$ is the mean of $f$ and $\hat{\bmu}$
is its empirical estimate based on $M$ samples, then 
for any $t>0$ we have
\[
\Pr\big[ || \mu-\hat{\bmu} ||^2 \ge t \big]   \le \frac{I_g}{M t}.
\]
\end{lemma}
\begin{proof} 
If $\bx_1, \ldots, \bx_M$ are independent draws from $f$, then 
\begin{eqnarray*}
\mathbf{E}[|| \mu  -\hat{\bmu} ||^2] 
 &=& \mathbf{E}\bigg[\big|\big|\mu  -\frac{\bx_1+ \ldots + \bx_M}{M}\big|\big|^2 \bigg] \\ 
 &=& \sum_{i=1}^M \frac{1}{M^2} 
  \mathbf{E}\bigg[\big|\big|\mu  - \bx_i \big|\big|^2 \bigg] 
  = \frac{I_g}{M},
\end{eqnarray*}
where the last inequality is by Lemma~\ref{l:var}.
Applying Markov's inequality on the left hand side, we get the stated claim. 
\end{proof}

\subsection{Transformation algorithm} \label{sec:transformation}


\begin{lemma}~\label{lem:transformation}
There is an algorithm \textsf{compute-transformation} such that given
access to samples from $f \in \CSI(c,d,g)$ and an error parameter
$0 < \epsilon < 1/2$, the algorithm takes $O(I_g)$ samples from $f$ and with
probability at least $9/10$ produces a vector $\tilde{\mu} \in \mathbb{R}^d$
and a real number $t$ with the following properties:
\begin{enumerate}
\item For $B_{t} = \{x : || x-\tilde{\mu} ||  \le \sqrt{t} \}$, we have 
$\Pr_{\bx \sim f} [ \bx \in B_t ] \geq 1 - \epsilon$.
\item 
$t = O(g^{-1}(\epsilon)^2)$,
\item  
For all $\kappa > 0$, the density $f_{B_t}$ satisfies
$\si(f_{B_t},\kappa) \leq \frac{4\epsilon}{\kappa} + 2 c$.
\end{enumerate}
\end{lemma}
\begin{proof}
For $M = 100 I_g$, the algorithm
\textsf{compute-transformation} simply works as follows: set
$\tilde{\mu}$ to be the empirical mean of the $M$ samples, and 
$t = 2 ((g^{-1}(\epsilon))^2 + 1/10)$. (Note that by Remark~\ref{remark:tail-weight} we have $t=\Theta(g^{-1}(\eps)^2).$).  Let $\mu$ denote the true mean of $f$.  First, 
by Lemma~\ref{lem:estimate-mean}, with
probability at least 0.9, the empirical
mean $\hat{\bmu}$ will be close to the true mean $\mu$ in the
following sense:
\begin{eqnarray}~\label{eq:estimate-mean}
|| \mu-\hat{\bmu} ||^2    \le \frac{1}{10}.
\end{eqnarray}
Let us assume for the rest of the proof that this happens; fix any such outcome and denote it $\tmu.$

We have 
\begin{eqnarray*}
|| x - \tmu ||^2
 \leq 2 (|| x - \mu ||^2 + || \mu - \tmu ||^2) 
 \leq 2 (|| x - \mu ||^2 + 1/10) \\
\end{eqnarray*}
and so
\begin{align*}
& \Pr_{\bx \in f} [ || \bx - \tmu ||^2 > t ] 
  \leq \Pr_{\bx \in f} [ 2 (|| \bx - \mu ||^2 + 1/10) > t ] 
 = \Pr[\|\bx - \mu\|^2 \geq g^{-1}(\eps)]
 \leq \epsilon.
\end{align*}

Applying Lemma~\ref{lem:restrict-si} completes the proof.
\end{proof}

The following proposition elaborates on the properties of the output
of the transformation algorithm.
 
\begin{lemma} \label{lem:properties}
Let $f \in \CSI(c,d,g)$, $\epsilon>0$, $\tilde{\mu} \in \mathbb{R}^d$,
and $t \in \mathbb{R}$ satisfy the properties stated in
Lemma~\ref{lem:transformation}. Consider the density
$f_{\mathrm{scond}}$ defined by
\begin{align*}
& f_{\mathrm{scaled}}(x) \eqdef 2\sqrt{t} \cdot f\big(2\sqrt{t} \cdot {( x+\tilde{\mu})}  \big) \\
& f_{\mathrm{scond}}(x) \eqdef f_{\mathrm{scaled}, B(1/2)}(x  ),
\end{align*}
where $f_{\mathrm{scaled}, B(1/2)}$ is the result of conditioning $f_{\mathrm{scaled}}$ on membership in $B(1/2)$.
Then the density $f_{\mathrm{scond}}(x)$ satisfies the following properties: 
\begin{enumerate}
\item The density $f_{\mathrm{scond}}$ is supported in the ball $B(1/2)$. 
\item For all $\epsilon < 1/2$
and $\kappa > 0$, the density $f_{\mathrm{scond}}$ 
satisfies
\[
\si(f_{\mathrm{scond}},\kappa) \leq \frac{4 \epsilon}{\kappa} + 4 c \sqrt{t}.\]
\end{enumerate}
\end{lemma}

\begin{proof}
First, it is easy to verify that function $f_{\mathrm{scond}}$ defined above is indeed a density. Item 1 
is enforced by fiat.
Now, for any direction $v$, we have
\begin{align*}
\si(f_{\mathrm{scaled}},v,\kappa) 
 & =
{\frac 1 \kappa} \cdot \sup_{\kappa' \in [0,\kappa]} \int_{\R^d} \left|f_{\mathrm{scaled}}(x+ \kappa'  v) - f_{\mathrm{scaled}}(x)\right| dx \\
 & =
{\frac{2\sqrt{t}}\kappa} \cdot \sup_{\kappa' \in [0,\kappa]} \int_{\R^d} \left|f(2 \sqrt{t} (x+ \kappa'  v)) - f(2 \sqrt{t} x)\right| dx. \\
\end{align*}
Using a change of variables, $u = 2 \sqrt{t} x$, we get
\begin{align}
\si(f_{\mathrm{scaled}},v,\kappa) 
 & = {\frac 1 \kappa} \cdot \sup_{\kappa' \in [0,\kappa]} \int_{\R^d} \left|f(u+ \kappa' 2 \sqrt{t}  v) - f(u)\right| du \nonumber \\
 & = {\frac 1 \kappa} \cdot \sup_{\kappa' \in [0,2 \sqrt{t} \kappa]} \int_{\R^d} \left|f(u+ \kappa'   v) - f(u)\right| du \nonumber \\
 & = 2 \sqrt{t} \cdot \si(f,v,2 \sqrt{t} \kappa) \le 2 c \sqrt{t}. \label{eq:a}
\end{align}
The last inequality uses that $f \in \CSI(c,d,g)$. Inequality (\ref{eq:a})  implies that $f_{\mathrm{scaled}} \in  \CSI(2c\sqrt{t},d,g)$.
Now, 
$\Pr_{\bx \sim f_{\mathrm{scaled}}}(\bx \in B(1/2)) = \Pr_{\bx \sim f}(\bx \in B_t) \geq 1 - \epsilon$, so
applying Lemma~\ref{lem:restrict-si} completes the proof.
\end{proof}

\subsection{Proof of Theorem~\ref{thm:no-noise}} \label{sec:proof-of-no-noise}

We are now ready to prove Theorem~\ref{thm:no-noise}.  Consider the following algorithm, which we call 
\textsf{construct-candidates}:

\begin{enumerate}

\item Run the transformation algorithm \textsf{compute-transformation} $D := O (\ln(1/\delta))$ many times (with parameter $\eps$ each time).  Let $(\tilde{\mu}^{(i)},t)$ be the output that it produces on the $i$-th run, where $t=O(g^{-1}(\epsilon)^2)$. 

\item For each $i \in [D]$,  let
$B_t^{(i)} = \{ x : || x - \tilde{\mu} || \leq \sqrt{t} \}$ and
$
f^{(i)}_{\mathrm{scond}}
$
be the density defined from $(\tilde{\mu}^{(i)},t)$ as in Lemma~\ref{lem:properties}.
\end{enumerate}

Before describing the third step of the algorithm, we observe that given the pair $(\tilde{\mu}^{(i)},t)$ it is easy to check whether any given $x \in \mathbb{R}^d$ belongs to $B_{t}^{(i)}$.\ignore{ Thus it is easy to efficiently obtain a $\pm \eps/2$-accurate estimate of $\Pr_{x \sim f}[x \in B_{t}]$ given draws from $f$; let $v^{(i)} \in [0,1]$ denote the estimate thus obtained.} We further make the following observations:

\begin{itemize}

\item If $\Pr_{\bx \sim f}[\bx \in B_{t}^{(i)}] \geq 1/2,$ then with probability at least $1/2$ a draw from $f$ can be used as a draw from $f_{B_{t}^{(i)}}$.  In this case, via rejection sampling, it is easy to very efficiently simulate draws from $f^{(i)}_{\mathrm{scond}}$ given access to samples from $f$ (the average slowdown is at most a factor of 2).  Note that if $(\tilde{\mu}^{(i)},t)$ satisfies the properties of Lemma~\ref{lem:transformation}, then $\Pr_{\bx \sim f}[\bx \in B_{t}^{(i)}] \geq 1-\eps$ and we fall into this case.

\item On the other hand, if $\Pr_{\bx \sim f}[\bx \in B_{t}^{(i)}] < 1/2,$ then it may be inefficient to simulate draws from $f^{(i)}_{\mathrm{scond}}.$  But any such $i$ will not satisfy the properties of Lemma~\ref{lem:transformation}, so if rejection sampling is inefficient to simulate draws from $f^{(i)}_{\mathrm{scond}}$ then we can ignore such an $i$ in what follows.

\end{itemize}

With this in mind, the third and fourth steps of the algorithm are as follows:
    
\begin{enumerate}

\item [3.] For each $i \in [D]$,\footnote{Actually, as described above, this and the fourth step are done only for those $i$ for which rejection sampling is not too inefficient in simulating draws from $f^{(i)}_{\mathrm{scond}}$ given draws from $f$; for the other $i$'s, the run of \textsf{learn-bounded} is terminated.} run the algorithm \textsf{learn-bounded} using $m$  samples from $f^{(i)}_{\mathrm{scond}}$, where $m=m(\eps,\delta,d)$ is the sample complexity of \textsf{learn-bounded} from Lemma~\ref{lem:finite-support}.
 Let ${h}_{\mathrm{scond}}^{(i)}$ be the resulting hypothesis that \textsf{learn-bounded} outputs.

\item [4.] Finally, for each $i \in [D]$ output the hypothesis obtained by 
inverting the mapping of Lemma~\ref{lem:properties}, i.e.
\begin{equation} \label{eq:hi}
{h}^{(i)}(x) \eqdef
{\frac 1 {2\sqrt{t}}} \cdot {h}_{\mathrm{scond}}^{(i)}\left(
{\frac 1 {2\sqrt{t}}} \cdot (x-\tilde{\mu}^{(i)})  \right).
\end{equation}
\end{enumerate} 

Thus the output of \textsf{construct-candidate} is a $D$-tuple of hypotheses $({h}^{(1)},
\dots, {h}^{(D)}).$ 

\medskip

We now analyze the \textsf{construct-candidate} algorithm.  Given Lemma~\ref{lem:transformation} and Lemma~\ref{lem:properties}, it is not difficult to show that with high probability at least one of the hypotheses that it outputs has error $O(\eps)$ with respect to $f$:

\begin{lemma} \label{lem:one-good}
With probability at least $1-O(\delta)$, at least one ${h}^{(i)}$ has $\int_x |{h}^{(i)}(x)-f(x)| dx \leq O(\eps).$
\end{lemma}
\begin{proof}
It is immediate from Lemma~\ref{lem:transformation} and the choice of $D$ that with probability $1-\delta$ at least one
triple $(\tilde{\mu}^{(i)},t)$ satisfies the properties of Lemma~\ref{lem:transformation}.  
 Fix $i'$ to be an $i$ for which this holds.

Given any $i \in [D]$, it is easy to carry out the check for whether rejection sampling is too inefficient in simulating $f^{(i)}_{\mathrm{scond}}$ in such a way that algorithm \textsf{learn-bounded} will indeed be run to completion (as opposed to being terminated) on $f^{(i')}_{\mathrm{scond}}$ with probability at least $1-\delta$, so we henceforth suppose that indeed \textsf{learn-bounded} is actually run to completion on $f^{(i')}_{\mathrm{scond}}$. 
Since $(\tilde{\mu}^{(i')},t)$ satisfies the properties of Lemma~\ref{lem:transformation}, by Lemma~\ref{lem:properties}, taking 
$\kappa = \min\{\eps{/2},\eps/(4 g^{-1}(\eps)c)\}$) the density $f^{(i')}_{\mathrm{scond}}$ satisfies  the required conditions for Lemma~\ref{lem:finite-support} to apply with that choice of $\kappa$.
The following simple proposition implies that ${h}^{(i)}$ is likewise $O(\eps)$-close to $f_{B_{t}}$:

\begin{proposition}~\label{prop:density-diff-transform}
Let $f$ and $g$ be two densities in $\mathbb{R}^d$ and let $x \mapsto A(x-z)$ be any invertible linear transformation over $\mathbb{R}^d$. Let $f_A(x) = \det(A) \cdot f(A(x-z))$ and $g_A(x) = \det(A) \cdot g(A(x-z))$ be the densities from $f$ and $g$ under this transformation.  Then $\dtv(f,g) = \dtv(f_A, g_A)$. 
\end{proposition}
\begin{proof}
\begin{align*}
\dtv(f_A,g_A) &= \int_{x} |f_A(x) - g_A(x)| dx = \int_x \det(A) |f(A(x-z)) - g(A(x-z)) | dx\\
& = \int_z |f(z) - g(z)| dz = \dtv(f,g), 
\end{align*}
where the penultimate equality follows by a linear transformation of variables. 
\end{proof}

It remains only to observe that by property 1 of Lemma~\ref{lem:transformation} the density $f_{B_{t}}$ is $\eps$-close to $f$, and then by the triangle inequality we have that ${h}^{(i)}$ is $O(\eps)$-close to $f$.  This gives Lemma~\ref{lem:one-good}.
\end{proof}

Tracing through the parameters, it is straightforward to verify that the sample and time complexities of \textsf{construct-candidates} are as claimed in the statement of Theorem~\ref{thm:no-noise}. These sample and time complexities dominate the sample and time complexities of the remaining portion of the algorithm, the hypothesis selection procedure discussed below.

All that is left is to identify a good hypothesis from the pool of $D$ candidates.   This can be carried out rather straightforwardly using well-known tools for hypothesis selection.  Many variants of the basic hypothesis selection procedure have appeared in the literature, see e.g.
\citep{Yatracos85,DaskalakisKamath14,AJOS14,DDS12stoc,DDS15}).  
The following is implicit in the proof of
Proposition 6 from \citep{DDS15}:

\begin{proposition} \label{prop:log-cover-size}
Let $\bD$ be a distribution
with support contained in a
set $W$
and let $\calD_\eps = \{ \bD_j\}_{j=1}^M$ be a collection of $M$ hypothesis distributions 
over $W$ with the property that there exists $i \in [M]$ such that
$\dtv(\bD,\bD_i) \leq \eps$.
There is an algorithm~$\mathrm{Select}^{\bD}$ which is given
$\eps$ and a confidence parameter $\delta$, and is provided
with access to
(i) a source of i.i.d. draws from $\bD$ and from $\bD_i$, for all $i \in [M]$;
and
(ii) a $(1+\beta)$ ``approximate evaluation oracle'' $\eval_{\bD_i}(\beta)$,
for each $i \in [M]$, which, on input $w \in W$, deterministically outputs $\tilde{D}_i^{\beta}(w)$ such that 
the value $\frac{\bD_i(w)}{1+\beta} \le \tilde{D}_i^{\beta}(w) \le (1+\beta) \cdot \bD_i(w)$. 
Further, $(1+\beta)^2 \le (1+\epsilon/8)$. 
 The $\mathrm{Select}^{\bD}$ algorithm has the following behavior:
It makes
$m = O\left( (1/ \eps^{2}) \cdot (\log M + \log(1/\delta)) \right)
$ draws from $\bD$ and {from} each $\bD_i$, $i \in [M]$,
and $O(m)$ calls to each oracle $\eval_{\bD_i}$, $i \in [M]$.  It runs in time $\poly(m,M)$ (counting each call to an $\eval_{\bD_i}$ oracle and draw from a $\bD_i$ distribution as unit time),
and with probability $1-\delta$ it outputs an index $i^{\star} \in [M]$ that satisfies $\dtv(\bD,\bD_{i^{\star}}) \leq 6\eps.$
\end{proposition}


As suggested above, the remaining step is to apply Proposition~\ref{prop:log-cover-size} to the list of candidate hypothesis ${h}^{(i)}$ which satisfies the guarantee of Lemma~\ref{lem:one-good}. However, to bound the sample and time complexity of running the procedure Proposition~\ref{prop:log-cover-size}, we need to bound the complexity  both of sampling from $\{{h}^{(i)}\}_{i \in [D]}$ as well as of constructing approximate evaluation oracles for these measures.\footnote{Note that while ${h}^{(i)}$ are forced to be non-negative and thus can be seen as measures, they need not integrate to $1$ and thus need not be densities.}
In fact, we will first construct densities out of the measures $\{{h}^{(i)}\}_{i \in [D]}$ and show how to both efficiently sample from these measures as well as construct approximate evaluation oracles for these densities. 

Towards this, let us now define $H_{\max}$ as follows: $H_{\max} = \max_{i \in [D]} \max_{z \in [-1,1]^n} {h}_{\mathrm{scond}}^{(i)}(z)$.  From Lemma~\ref{lem:finite-support} (recall that Lemma~\ref{lem:finite-support} was applied with $\kappa = \min\{\eps{/2},\eps/(4g^{-1}(\eps)c)\}$) we get that
\begin{equation}~\label{eq:H-max-val}H_{\max}=
  O_{c,d}\left( \left( {\frac {g^{-1}(\eps)}{\eps}} \right)^d \log^{2d} \frac{g^{-1}(\eps)}{\eps} \right).
\end{equation}
We will carry out the rest of our calculations in terms of $H_{\max}$.

\begin{observation}~\label{obs:h-compute}
For any $i \in [D]$, $\int_{x \in [-1,1]^d} {h}_{\mathrm{scond}}^{(i)} (x) dx$ can be 
estimated to additive accuracy $\pm \epsilon$  and confidence $1-\delta$  in time $O_d\left( \frac{H_{\max}^2}{\epsilon^2} \cdot \log (1/\delta)\right)$. 
\end{observation}
\begin{proof}
First note that it suffices to estimate
the quantity  $\mathbf{E}_{x \in [-1,1]^d} 
[{h}_{\mathrm{scond}}^{(i)} (x)]$ to additive error $\epsilon/2^d$. However, this can be estimated using the trivial random sampling algorithm. In particular,  as 
${h}_{\mathrm{scond}}^{(i)}(x) \in [0,H_{\max}]$, the variance of the simple unbiased estimator for $\mathbf{E}_{x \in [-1,1]^d} 
[{h}_{\mathrm{scond}}^{(i)} (x)]$ is also bounded by $H_{\max}^2$. This finishes the proof.  
\end{proof}

Note that, while the algorithm of Observation~\ref{obs:h-compute} does random sampling, this sampling
is not from $f$, so it adds nothing to the sample complexity of the learning algorithm.

Next, for $i \in [D]$, let us define the quantity $Z_i$ to be $Z_i = \int_{x\ignore{\in [-1,1]^d}} {h}^{(i)}(x) dx$. Since the functions ${h}^{(i)}$  and ${h}_{\mathrm{scond}}^{(i)}$ are obtained from each other by linear transformations (recall (\ref{eq:hi})),
we get that that 
\[
2 \sqrt{t} Z_i = \int_{x\ignore{\in [-1,1]^d}} {h}_{\mathrm{scond}}^{(i)}\left({\frac 1 {2\sqrt{t}}} \cdot (x-\tilde{\mu}^{(i)})\right) dx. 
\]
We now define the functions ${H}^{(i)}$ and ${H}_{\mathrm{scond}}^{(i)}$ as
\[
{H}^{(i)} (x) = \frac{{h}^{(i)}(x)}{Z_i} \ \ \textrm{and} \  \ {H}_{\mathrm{scond}}^{(i)}(x) = \frac{{h}_{\mathrm{scond}}^{(i)}(
{\frac 1 {2\sqrt{t}}} \cdot (x-\tilde{\mu}^{(i)}))}{Z_i} \cdot {\frac 1 {2\sqrt{t}}} .
\]
Observe that the functions ${H}^{(i)}$ and ${H}_{\mathrm{scond}}^{(i)}$ are densities (i.e. they are non-negative and integrate to $1$). 
First, we will show that it suffices to run the procedure $\mathrm{Select}^{\bD}$ on the densities ${H}^{(i)}$. 
To see this, note that  Lemma~\ref{lem:one-good} says that there exists $i \in [D]$ such that $h^{(i)}$ satisfies $\int_x |h^{(i)}(x) - f(x)| = O(\epsilon)$. For such an $i$, $Z_i \in [1-O(\epsilon), 1+ O(\epsilon)]$. 
\ignore{\pnote{The above change regarding what we know about $Z_i$ propagates -- the propagated changes are not marked in blue.}}
Thus, we have the following corollary. 
\begin{corollary}~\label{corr:one-good}
With probability at least $1-\delta$, at least one ${H}^{(i)}$ satisfies $\int_x |{H}^{(i)}(x) - f(x)| = O(\epsilon)$. Further, for such an $i$, $Z_i \in [1-O(\epsilon), 1+O(\epsilon)]$. 
\end{corollary}
Thus, it suffices to run the procedure $\mathrm{Select}^{\bD}$ on the candidate distributions $\{{H}^{(i)}\}_{i \in [D]}$. The next proposition shows that 
the densities $\{  {H}^{(i)}\}_{i \in [D]}$ are samplable.

\begin{proposition}~\label{prop:samplable}
A draw from the density ${H}^{(i)}(x)$ can be sampled in time
$O(H_{\max}/Z_{i})$.
\end{proposition}
\begin{proof}
First of all, note that it suffices to sample from ${H}_{\mathrm{scond}}^{(i)}$ since $H^{(i)}$ and $H^{(i)}_{\mathrm{scond}}$ are linear transformations of each other. However, sampling from ${H}_{\mathrm{scond}}^{(i)}$  is easy using rejection sampling. More precisely, the  distribution ${H}_{\mathrm{scond}}^{(i)}$ is supported on $[-1,1]^d$. We sample from ${H}_{\mathrm{scond}}^{(i)}$ as follows: 
\begin{enumerate}
\item Let $C = [-1,1]^d \times [0, H_{\max}]$. Sample a  uniformly random point $z'= (z_1, \ldots, z_{d+1})$ from $C$. 
\item If $z_{d+1} \le {h}_{\mathrm{scond}}^{(i)}(z_1, \ldots, z_d)$, then return the point $z=(z_1, \ldots, z_d)$. 
\item Else go to Step 1 and repeat. 
\end{enumerate}
Now note that conditioned on returning a point  in step $2$, the point $z$ is returned with probability proportional to ${h}_{\mathrm{scond}}^{(i)}(z)$. Thus, the distribution sampled by this procedure is indeed ${H}_{\mathrm{scond}}^{(i)}(z)$. To bound the probability of success, note that the total volume of $C$ is $2^{d} \times H_{\max}$. On the other hand, step $2$ is successful only if $z'$ falls in a region of volume $Z_{i}$. This finishes the proof.  
\end{proof}
The next proposition says that if $Z_i \ge 1/2$, then there is an approximate evaluation oracle for the density ${H}^{(i)}$. 
\begin{proposition}~\label{prop:density-compute}
Suppose $Z_i \ge 1/2$. Then there is a $(1+O(\epsilon))$- approximate evaluation oracle for ${H}^{(i)}$ which can be computed at any point $w$ in time $O\left( \frac{H_{\max}^2 }{\epsilon^2} \right)$.
\end{proposition}
\begin{proof}
Note that we can evaluate ${h}^{(i)}$ at any point $w$ exactly and thus the only issue is to estimate the normalizing factor $Z_i$. Note that since $Z_i \ge 1/2$ , estimating $Z_i$ to within an additive $O(\epsilon)$ gives us a $(1+O(\epsilon))$ multiplicative approximation to $Z_i$ and hence to ${H}^{(i)}(w)$ at any point $w$. However, by Observation~\ref{obs:h-compute}, this takes time $O \left( \frac{H_{\max}^2 }{\epsilon^2} \right)$, concluding the proof. 
\end{proof}

We now apply Proposition~\ref{prop:log-cover-size} as follows. 
\begin{enumerate}
\item For all $i \in [D]$, estimate $Z_i$ using Observation~\ref{obs:h-compute} up to an additive error $\epsilon$. Let the estimates be $\hat{Z}_i$. 
\item Let us define $S_{\mathrm{feas}} = \{i \in [D]: \hat{L}_i \ge 1/2\}$. 
\item We run the routine \textsf{Select}$^{\mathbf{D}}$ on the densities $\{{H}^{(i)}\}_{i \in S_{\mathrm{feas}}}$. 
 To sample from a density ${H}^{(i)}$, we use Proposition~\ref{prop:samplable}. We also construct a $\beta = \epsilon/32$ approximation oracle
for each of the densities ${H}^{(i)}$ using Proposition~\ref{prop:density-compute}. Return the output of  \textsf{Select}$^{\mathbf{D}}$. 
\end{enumerate}
The correctness of the procedure follows quite easily. 
Namely, note that Corollary~\ref{corr:one-good} implies that there is one $i$ such that both $Z_i \in [1-O(\epsilon), 1+O(\epsilon)]$ and $\int_x |{H}^{(i)}(x) - f(x)| = O(\epsilon)$. Thus such an $i$ will be in $S_{\mathrm{feas}}$. Thus, by the guarantee of \textsf{Select}$^{\mathbf{D}}$, the output hypothesis is $O(\epsilon)$ close to $f$. 

We now bound the sample complexity and time complexity of this hypothesis selection portion of the algorithm. First of all, the number of samples required from $f$ for running  \textsf{Select}$^{\mathbf{D}}$ is  $O((1/\epsilon^2) \cdot (\log (1/\delta)  +d^2 \log d + \log \log (1/\delta)) = O((1/\epsilon^2) \cdot (\log (1/\delta)  +d^2 \log d )$. This is clearly dominated by the sample complexity of the previous parts. To bound the time complexity, note that the time complexity of invoking the sampling oracle for any ${H}^{(i)}$  ($ i \in S_{\mathrm{feas}}$) is dominated by the time complexity of the approximate oracle which is $2^{O(d)} \cdot H_{\max}^2/\epsilon^2$.  
The total number of calls to the sampling as well as evaluation oracle is upper bounded by $\frac{1}{\epsilon^2} (D \log D + D \log (1/\delta))$. Plugging in the value of $H_{\max}$ as well as $D$, we see that the total time complexity is dominated by the bound in the statement of Theorem~\ref{thm:no-noise}. This finishes the proof.


\section{Noise-tolerant density estimation for $\CSI(c,d,g)$}
\label{sec:semi-agnostic}

Fix any nonincreasing tail bound function $g: \R^+ \to [0,1]$ which satisfies $\lim_{t \to +\infty} g(t) = 0$
and the condition of Remark~\ref{remark:tail-weight}
and any constant $c \geq 1$.
In this section we prove  our main result, Theorem~\ref{thm:main-semiagnostic}, which gives a noise tolerant density estimation algorithm for
$\CSI(c,d,g)$. We first recall the precise model of noise we consider in this paper. 
\begin{definition}~\label{def:corruption}
For two densities $f$ and $f' \in \mathbb{R}^d$, we say that $f'$ is an $\epsilon$-corruption of $f$ if $f'$ can be expressed as $f' = (1-\epsilon) \cdot f + \epsilon \cdot f_{err}$ where $f_{err}$ is a density in $\mathbb{R}^d$. 
\end{definition}
This model of noise is sometimes referred to as \emph{Huber's contamination model}~\citep{huber1967behavior}.

\begin{theorem} [Noise-tolerant density estimation for $\CSI(c,d,g)$.] \label{thm:main-semiagnostic}
For any $c,g$ as above and any $d \geq 1$,  there is an algorithm with the following property:  Let $f$ be any density 
(unknown to the algorithm) which belongs to $\CSI(c,d,g)$ and let $f'$ be an $\epsilon$-corruption of $f$. 
Given $\eps$ and any confidence parameter $\delta > 0$ and access to independent draws from $f'$, the algorithm with probability $1-O(\delta)$ outputs a hypothesis  ${h}:[-1,1]^d \rightarrow \mathbb{R}^{\geq 0}$ such that $\int_{x \in \mathbb{R}^d} |f'(x) - {h}(x)| \le O(\epsilon)$. 

The algorithm runs in time 
\ignore{
\[
O\left( 
 \exp(O(I_g)) \cdot
  \left(
   (1 + g^{-1}(\epsilon))^{2d}
          \left( \frac{1}{\epsilon} \right)^{2 d + 2}
          \log^{4 d} \left( \frac{1 + g^{-1}(\epsilon)}{\epsilon} \right)
          \log \left( \frac{1 + g^{-1}(\epsilon)}{\epsilon \delta} \right) 
  + I_g \right) \log \frac{1}{\delta}
\right)
\] 
}
\[
O_{c,d}\left( 
 \exp(O(I_g)) \cdot
  \left(
   (g^{-1}(\epsilon))^{2d}
          \left( \frac{1}{\epsilon} \right)^{2 d + 2}
          \log^{4 d} \left( \frac{ g^{-1}(\epsilon)}{\epsilon} \right)
          \log \left( \frac{ g^{-1}(\epsilon)}{\epsilon \delta} \right) 
  + I_g \right) \log \frac{1}{\delta}
\right)
\] 
and uses 
\ignore{
\[
O \left(
 \exp(O(I_g)) \cdot
 \left(
 (1 + g^{-1}(\epsilon))^{d}
          \left( \frac{1}{\epsilon} \right)^{d+2}
          \log^{2d} \left( \frac{1 + g^{-1}(\epsilon)}{\epsilon} \right)
          \log \left( \frac{1 + g^{-1}(\epsilon)}{\epsilon \delta} \right)
 + I_g \right) \log \frac{1}{\delta}
  \right)
\]
}
\[
O_{c,d} \left(
 \exp(O(I_g)) \cdot
 \left(
 ( g^{-1}(\epsilon))^{d}
          \left( \frac{1}{\epsilon} \right)^{d+2}
          \log^{2d} \left( \frac{ g^{-1}(\epsilon)}{\epsilon} \right)
          \log \left( \frac{ g^{-1}(\epsilon)}{\epsilon \delta} \right)
 + I_g \right) \log \frac{1}{\delta}
  \right)
\] 
samples.  
\end{theorem}

Theorem~\ref{thm:main-semiagnostic} is identical to Theorem~\ref{thm:no-noise} except that now the target density from which draws are received is $f'$, which is an $\eps$-corruption of $f \in \CSI(c,d,g)$, rather than $f$ itself,
and the dependence on $I_g$ is exponential.   
(On the other hand, if $f$ is isotropic,
then recalling Lemma~\ref{l:var} the function $g$ can be taken to be such that $I_g = \E_{\bx \sim f} [ || \bx - \mu ||^2 ] = d$, so that 
$\exp(O(I_g)) = O(1)$ and the complexity is
the same as in the noise-free case.)
This requires essentially no changes in the algorithm and  fortunately most of the analysis from earlier can also be reused in a fairly black-box way. We briefly  explain how the analysis of Section~\ref{sec:no-noise} can be augmented to handle having access to draws from $f'$ rather than $f$.

\begin{proofof}{Theorem~\ref{thm:main-semiagnostic}}
Without loss of generality, we can assume that $\epsilon \leq 1/10$, as otherwise any hypothesis is trivially $O(\eps)$-close to the target density. 
Recall that $f'$ can be expressed as $f' = (1-\epsilon) \cdot f + \epsilon \cdot f_{err}$ for some density $f_{err} \in \mathbb{R}^d$.  Since $\eps \le 1/10$, this means that with probability $9/10$, a random sample from $f'$ is in fact distributed exactly as a random sample from $f$. We now revisit the steps in the algorithm \textsf{construct-candidates} (Proof of Theorem~\ref{thm:no-noise}) and briefly sketch why sample access to $f'$ instead of $f$ suffices.

\textbf{Step 1:} Note that each invocation of the algorithm \textsf{compute-transformation} is supposed to draw $O(I_g)$ samples from $f$.  We have access to $f'$ rather than $f$, but since each sample from $f'$, with probability at least $9/10$, is a sample from $f$,  with probability at least $\exp(-O(I_g))$, a run of \textsf{compute-transformation} with samples from $f'$ is the same as a run with samples from $f$.   So now in Step~1, the algorithm \textsf{compute-transformation} is run $\exp(-O(I_g)) \cdot \ln(1/\delta))$ many times rather than $O(\ln(1/\delta))$ many times as in the original version (this accounts for the additional $\exp(O(I_g))$ factor in the bounds of Theorem~\ref{thm:main-semiagnostic} versus Theorem~\ref{thm:no-noise}).

\textbf{Step 2:} This goes exactly as before with no changes. In particular, for every $i$, if ${f}_{\mathrm{trans}}^{(i)}$ is the true density obtained by the $i^{th}$ transformation, then we have sample access to ${f'}_{\mathrm{trans}}^{(i)}$  which is an $\epsilon$-corruption of ${f}_{\mathrm{trans}}^{(i)}$. 

\textbf{Step 3:} In Step $3$, we run the routine \textsf{learn-bounded} as usual. In particular, let us assume that ${f}_{\mathrm{trans}}^{(i)}$ satisfies the  conditions of Lemma~\ref{lem:finite-support}. Then Corollary~\ref{corr:agnostic}, which established noise-tolerance of \textsf{learn-bounded}, implies that with sample access to ${f'}_{\mathrm{trans}}^{(i)}$, the resulting hypothesis ${h'}_{\mathrm{trans}}^{(i)}$ is $2\epsilon$-close to ${f}_{\mathrm{trans}}^{(i)}$. 

It is easy to verify that Step 4 and the subsequent steps in hypothesis testing can go on exactly as before with sample access to $f'$ instead of $f$. In particular, the hypothesis testing routine $\mathrm{Select}^D$ will output a hypothesis ${h}^{(i)}$ which is $2\epsilon$ close to $f$. 
\end{proofof}


\section{Efficiently learning multivariate log-concave densities} \label{sec:logconcave}

In this section we present our main application, which is an efficient algorithm for noise-tolerantly learning $d$-dimensional log-concave densities.  We prove the following:

\begin{theorem} [Restatement of Theorem~\ref{thm:log-concave-intro}] \label{thm:log-concave}
There is an algorithm with the following property:  Let $f$ be a unknown log-concave density over $\mathbb{R}^d$ and let $f'$ be an $\eps$-corruption of $f$.\ignore{ whose covariance matrix $\Sigma$ has full rank}\ignore{\rnote{Do we need this assumption?  

Suppose first that we were only claiming a noiseless result.  If the log-concave distribution could be lower dimensional, I guess we could run a pre-processing step which uses $m$ samples to identify a lower-dimensional subspace $A$ containing almost all (has to be at least $1-\eps$) of the mass, if one exists, and then run the learning algorithm on ($f$ restricted to $A$) over $\mathbb{R}^{|A|}$.  But we are trying to handle noise.

If there is noise (the semi-agnostic setting), I guess we can draw $m$ samples and guess all possible noisy subsets, and thus identify the low-dimensional subspace $A$ where the noiseless distribution lives; and then as above, run the learning algorithm on ($f$ restricted to $A$).  Note that any noisy examples which lie outside the subspace $A$ will be discarded.  This should work, though it requires an additional step of hypothesis testing because of guessing the noisy subsets.

What will the overhead of doing this guessing be?  I guess we have to try all subsets of size $\eps m$ out of the $m$-element sample, so this incurs running time ${m \choose \eps m} \approx 2^{H(\eps) m}$ at least. How big does $m$ need to be for the preprocessing to work?  If $m=d/\eps$ then this would be $(1/\eps)^d$ which is fine.  If $m=d/\eps^2$ then this is $(1/\eps)^{1/\eps}$ which is too expensive for us.}} 
Given any error parameter $\eps>0$ and confidence parameter $\delta > 0$ and access to independent draws from $f'$, the algorithm with probability $1-\delta$ outputs a hypothesis density ${h}:\mathbb{R}^d \rightarrow \mathbb{R}^{\geq 0}$ such that $\int_{x \in \mathbb{R}^d} |f'(x) - {h}(x)| \le O(\epsilon)$.  
The algorithm runs in time 
\[
O_d
\left(
  \left(
          \frac{1}{\epsilon} \right)^{2 d + 2}
          \log^{7 d} \left( \frac{1}{\epsilon} \right)
          \log \left( \frac{1}{\epsilon \delta}
  \right) \log \frac{1}{\delta}
\right)
\] 
and uses 
\[
O_d
\left( 
          \left( \frac{1}{\epsilon} \right)^{d+2}
          \log^{4d} \left( \frac{1}{\epsilon} \right)
          \log \left( \frac{1}{\epsilon \delta} 
 \right) \log \frac{1}{\delta}
\right)
\] samples.  
\end{theorem}

We will establish Theorem~\ref{thm:log-concave} in two stages.
First, we will show that  any log-concave $f$ that is nearly
isotropic in fact belongs to a suitable class $\CSI(c,d)$; given
this, the theorem follows immediately from
Theorem~\ref{thm:main-semiagnostic} and a straightforward tracing
through of the resulting time and sample complexity bounds.  Then,
we will reduce to the near-isotropic case, similarly to what was
done in \citep{LovaszVempala07,balcan2013active}.

First, let us state the theorem for the well-conditioned case.  For this, the following definitions
will be helpful.
\begin{definition}~\label{def:PSD-aprx}
 Let $\Sigma$ and $\tilde{\Sigma}$ be two positive semidefinite matrices.  We say that $\Sigma$ and $\tilde{\Sigma}$ are 
 \emph{$C$-approximations}
of each other (denoted by $\Sigma \approx_C \tilde{\Sigma}$) if for every $x \in \mathbb{R}^n$ such that $x^T \widetilde{\Sigma} x \neq 0$, we have
$$
\frac1C \le \frac{x^T \Sigma x}{x^T \widetilde{\Sigma} x} \le C. 
$$
\end{definition}
\begin{definition}~\label{def:C.round}
Say that the probability distribution is $C$-nearly-isotropic if
its covariance matrix $C$-approximates $I$, the $d$-by-$d$ identity matrix.
\end{definition}

\begin{theorem} \label{thm:log-concave.round}
There is an algorithm with the following property:  Let $f$ be a unknown $C$-nearly-isotropic log-concave density over $\mathbb{R}^d$ and let $f'$ be an $\eps$-corruption of $f$, where $C$ and $d$ are constants.

Given any error parameter $\eps>0$ and confidence parameter $\delta > 0$ and access to independent draws from $f'$, the algorithm with probability $1-\delta$ outputs a hypothesis density ${h}:\mathbb{R}^d \rightarrow \mathbb{R}^{\geq 0}$ such that $\int_{x \in \mathbb{R}^d} |f'(x) - {h}(x)| \le O(\epsilon)$.  
The algorithm runs in time 
\[
O_{C,d}
\left(
          \left( \frac{1}{\epsilon} \right)^{2 d + 2}
          \log^{7 d} \left( \frac{1}{\epsilon}  \right) 
          \log \left( \frac{1}{\epsilon \delta}
  \right) \log \frac{1}{\delta}
\right)
\] 
and uses 
\[
O_{C,d}
\left( 
          \left( \frac{1}{\epsilon} \right)^{d+2}
          \log^{4d} \left( \frac{1}{\epsilon} \right)
          \log \left( \frac{1}{\epsilon \delta} 
 \right) \log \frac{1}{\delta}
\right)
\] samples.  
\end{theorem}

By Theorem~\ref{thm:main-semiagnostic}, Theorem~\ref{thm:log-concave.round} is an immediate consequence of the following theorem on the
shift-invariance of near-isotropic log-concave distributions.
\begin{theorem}
\label{t:lc_shift_invariant}
Let $f$ be a $C$-nearly-isotropic log-concave density in $\mathbb{R}^d$,
for constants $C$ and $d$.  Then, for $g(t) = e^{-\Omega(t)},$ 
there is a constant $c_1=O_{C,d}(1)$ such that $f \in \CSI(c_1, d, g)$.
\end{theorem}
\begin{proof}
The fact that $f$ has $e^{-\Omega(t)}$-light tails directly follows from Lemma~5.17 of
\citep{LovaszVempala07}, so it remains to prove that there is
a constant $c_1$ such that
$f \in \CSI(c_1, d)$.  
Because membership in $\CSI(c_1, d)$ requires that a condition
be satisfied for all directions $v$, rotating a distribution does not affect its membership
in $\CSI(c_1, d)$.

Choose a unit vector $v$ and $\kappa>0$.  
By rotating the distribution if necessary, 
we may assume that $v=e_1$, and our goal of showing that $\si(f,e_1,\kappa) \leq c_1$
is equivalent to showing that
\begin{equation}\label{eq:AAA}
\int |f(x) - f(x + \kappa' e_1)| dx \leq c_1 \kappa
\end{equation}
for all $\kappa' \leq \kappa.$

We bound the integral of the LHS as follows.  Fix some value of $x' \eqdef (x_2, \ldots, x_d)$. 
Let us define $L_{x'} \eqdef \{ (x_1, x_2, \ldots, x_d) : x_1 \in \R \}$
to be the line through $(0,x_2, \ldots, x_d)$ and $(1,x_2, \ldots, x_d)$.
Since the restriction of a concave function to a line is concave, the
restriction of a log-concave distribution to a line is log-concave.  Since
\begin{equation}
\label{e:reduce.to.1d}
\int |f(x) - f(x + \kappa' e_1)|\; dx
 = \int_{x'} \int_{x_1} |f(x_1,x_2,...,x_d) - f(x_1 + \kappa', x_2,...,x_d)| \; dx_1 dx'
\end{equation}
we are led to examine the one-dimensional log-concave measure
$f(\cdot,x_2,...,x_d)$.
The following will be useful for that.
\begin{claim}~\label{clm:integral}
Let $\ell: \mathbb{R} \rightarrow \mathbb{R}$ be a log-concave measure. Then, $$\int |\ell(t) - \ell(t+h) | dt \le 3h \cdot \max_{t \in \R} \ell(t).$$
\end{claim}
\begin{proof}
Log-concave measures are unimodal (see \citep{ibragimov1956composition}).
Let $z$ be the mode of $\ell$, so that $\ell$ is non-decreasing 
on the interval $[-\infty,z]$ and non-increasing in $[z,\infty]$. 
We have
\begin{align*}
& \int |\ell(t) - \ell(t+h) | \; dt \\
& = \int_{-\infty}^{z-h} |\ell(t) - \ell(t+h) | \; dt 
       + \int_{z-h}^{z} |\ell(t) - \ell(t+h) | \; dt 
       + \int_{z}^{\infty} |\ell(t) - \ell(t+h) | \; dt \\
& = \int_{-\infty}^{z-h} \ell(t+h) - \ell(t) \; dt 
       + \int_{z-h}^{z} |\ell(t) - \ell(t+h) | \; dt 
       + \int_{z}^{\infty} \ell(t) - \ell(t+h) \; dt \\
& \hspace{4.5in} \mbox{(since $z$ is the mode of $\ell$)} \\
& = \int_{z-h}^{z} \ell(t) \; dt 
       + \int_{z-h}^{z} |\ell(t) - \ell(t+h) | \; dt 
       + \int_{z}^{z+h} \ell(t) \; dt \\
& \leq 3 h \max_{t \in \R} \ell(t).
\end{align*}
\end{proof}

Returning to the proof of Theorem~\ref{t:lc_shift_invariant}, applying
Claim~\ref{clm:integral} with (\ref{e:reduce.to.1d}), we get
\begin{equation} \label{eq:potato}
\int |f(x) - f(x + \kappa' e_1)|\; dx
 \le 3 \kappa'
 \int_{x'} \left(\max_{x_1 \in L_{x'}} f(x_1,x')\right) \; dx'.
\end{equation}
Now, since an isotropic log-concave distribution $g$ satisfies
$g(x) \leq K \exp(-\| x \|)$ for an absolute constant $K$ (see Theorem~5.1 of \cite{SaumardWellner14}) , our $C$-nearly-isotropic log-concave distribution $f$ satisfies $f(x) \leq C^d K \exp(-\|x\|) = O_{C,d}(\exp(-\|x\|))$.\ignore{\pnote{How
about this treatment, which emphasizes the fact that $C$ and $d$ are
constants more?  ($C$ and $d$ are also removed from the O below.)}}
Plugging this into (\ref{eq:potato}), we get
\begin{align*} \label{eq:tuber}
\int |f(x) - f(x + \kappa' e_1)|\; dx
 &\le O_{C,d}( \kappa' )
 \int_{x'} \left(\max_{x_1 \in L_{x'}} \exp(-\|(x_1,x')\|)\right) \; dx'\\
 &\le O_{C,d}( \kappa' )
 \int_{x'} \exp(-\|x'\|) \; dx'.
\end{align*}
\ignore{Applying the fact \citep{LovaszVempala07} that
$\max_x \ell(x) \leq C^d 2^{8d} \cdot d^{d/2} = O(1)$ completes the proof.}
Since the integral converges, this finishes the proof.
\end{proof}


Finally, we turn to the problem of
learning log-concave distributions that are not $C$-nearly-isotropic. To learn in this case, we need to
rescale the axes if necessary to transform the distribution so that 
it is $C$-nearly-isotropic. To compute this rescaling, we will first assume that there is no noise present in the samples. (The trick to handle noise is simple and will be discussed later.)
First, as has often been observed, we may assume without loss of
generality that the covariance matrix of the target has full rank,
since, otherwise the algorithm can efficiently find the
affine span of the entire distribution (possibly up to a negligible amount of probability mass), and the algorithm can be carried out within that lower dimensional subspace.
To bring the distribution to nearly isotropic position, we will be using ideas from 
\citep{LovaszVempala07}.  (We require the additional analysis below, rather than invoking their results as a black box, to cope with the
fact that the mean is unknown.)

Our starting point is the
following lemma due to 
Lov{\'{a}}sz and Vempala \cite{LovaszVempala07}.
\begin{lemma}~\label{lem:Rudelson}
Let $f$ be a zero-mean log-concave density on $\mathbb{R}^d$.  For $m
= O(d \log^3 d)$, if $\Sigma$ denotes the (population) covariance
matrix of $f$ and $\hat{\mathbf{\Sigma}}$ is the empirical covariance matrix
from $m$ samples of $f$, then, with probability $9/10$,
$\hat{\mathbf{\Sigma}}$ is a $11/10$ approximation to $\Sigma$.
\end{lemma}

Lemma~\ref{lem:Rudelson} enables us to estimate the covariance
matrix if we know the mean.  To apply it when we do not, we appear
to need an estimate of the mean that is especially good in
directions with low variance.  The following is aimed at obtaining such an estimate.

Recall that, for a set $A$ of real-valued functions on a common domain
$X$, the {\em pseudo-dimension} of $A$,
which is denoted by $\mathrm{Pdim}(A)$,
is the VC-dimension of the set
of indicator functions, one for each $a \in A$, for whether
$(x,y)$ satisfies $a(x) \geq y$.
We will use the following standard VC bound.
\begin{lemma}[\cite{Tal94}]
\label{l:vc}
For any set $A$ of functions with a common domain $X$ and ranges
contained in $[-M,M]$, for any distribution $D$, 
$m = O(\frac{M^2( \mathrm{Pdim}(A) + \log(1/\gamma))}{\epsilon^2})$
suffices for a set $S$ of $m$ examples 
drawn according to $D$, with probability $1-\gamma$, 
to have all $a \in A$ have 
\[
\left| \E_{\bx \sim D}(a(\bx)) - \frac{1}{m} \sum_{x' \in S} a(x') \right|
\leq \epsilon.
\]
\end{lemma}

The proof of the following lemma follows a similar lemma in \citep{KLS:09jmlr}.
\begin{lemma}
\label{l:pdim.mean}
Fix 
a function $b$ from
$\R^d$ to $\R^+$.  Define $a_u = b(u) \cdot (u \cdot x)$.
The pseudo-dimension of $\{ a_u : u \in B(1) \}$ is $O(d)$.
\end{lemma}
\begin{proof} Any $(x,y)$ satisfies 
$a_u(x) \geq y$ iff
\[
b(u) (u \cdot x) \geq y.
\]
Thus, the set of indicator functions
for $a_u(x) \geq y$ can be embedded into the set of homogeneous halfspaces
over $\R^{d+1}$, which is known to have VC-dimension $O(d).$
\end{proof}

Now we are ready for the result we require on estimating the mean:

\begin{lemma}
\label{l:mean}
Fix any log-concave distribution $f$ over $\R^d$
and any
$\alpha > 0$.
For $m = O\left(\frac{d \log^2(d/\alpha)}{\alpha^2}\right)$, 
 with probability at least $3/4$,
a multiset $S$ of $m$ samples drawn i.i.d. from $f$ 
satisfies, for all unit length $u$,
\begin{equation}
\label{e:mean}
\frac{| \E_{\bx \sim S}( u \cdot \bx) - \E_{\bx \sim f}(u \cdot \bx)|}
     {\sqrt{\Var_{\bx \sim f}(u \cdot \bx)}}
 \leq \alpha.
\end{equation}
\end{lemma}
\begin{proof} Translating the distribution $f$ translates both $\E_{\bx
  \sim S}( u \cdot \bx)$ and $\E_{\bx \sim f}( u \cdot \bx)$ the same way,
and does not affect $\Var_{\bx \sim f}(u \cdot \bx)$, so we may assume without loss
of generality that $f$ has zero mean.  
Let $f_B$ be the distribution
obtained from $f$ by conditioning the choice of $x$ on the event that
$| u \cdot x | \leq \sqrt{\Var_{\bx \sim f}(u \cdot \bx)} \cdot \frac{\ln (8 m)}{c}$ for all unit length $u$, where $c$ is a large constant.
Lemma~5.17 of \citep{LovaszVempala07} implies that, for large
enough $c$, the total variation
distance between $f$ and $f_B$ is $1/(8m)$, so that the total variation
distance between $m$ draws from $f$ and $m$ draws from $f_B$ is at most
$1/8$.  We henceforth assume that the $m$ draws from $f$ are in fact drawn from $f_B$, and proceed to analyze $f_B$.

For any unit length $u$, define $a_u$ by 
$a_u(x) = \frac{| u \cdot x |}{\sqrt{\Var_{\bx \sim f}(u \cdot \bx)}}$.
Lemma~\ref{l:pdim.mean} implies that $\{ a_u : u \in B(1) \}$
has pseudo-dimension $O(d)$.  Furthermore, when $x$ is chosen
from the support of $f_B$, each $a_u$ takes values in 
an interval of size $O(\log m)$.  Thus we may apply Lemma~\ref{l:vc} to obtain
Lemma~\ref{l:mean}. 
\end{proof}

Now we are ready to present and analyze the transformation.

\begin{lemma}~\label{lem:scaling}
There is an algorithm \textsf{rescale} such that given
access to samples from a log-concave distribution $f$, and an error
parameter $\epsilon>0$, the algorithm takes $O(d \log^3 d)$ samples
from $f$ and with probability at least $1/2$
produces a non-singular positive definite matrix $\tilde{\Sigma} \in \mathbb{R}^{d \times d}$
such that, if $\Sigma$ is the covariance matrix of $f$, 
for any unit vector $v$, 
\[
\frac{1}{2} \leq \frac{v^{T} \Sigma v}{v^T \tilde{\Sigma} v} 
   \leq 2.
\]
\end{lemma}
\begin{proof}
For a large constant $C$ and $M = C d \log^3 d$
the algorithm \textsf{rescale} first uses $M$ examples to
construct an estimate $\tilde{\bmu}$ of the mean of $f$, and then
uses $\tilde{\bmu}$ to use the examples estimate the covariance matrix.

Lemma~\ref{l:mean} implies that, if $C$ is large enough, then 
with probability $3/4$, for all unit length $v$, we have
\begin{equation}
\label{e:means}
\frac{| \mu \cdot v - \tilde{\bmu} \cdot v|}
     {\sqrt{\Var_{\bx \sim f}[v \cdot \bx]}}
   \leq 
   \frac{1}{10}.
\end{equation}
Lemma~\ref{lem:Rudelson} implies that, with probability $3/4$ over a random i.i.d.~draw of $\bx_1,\dots,\bx_M \sim f,$ we have
\begin{equation}
\label{e:rudelson}
9/10 \leq \frac{\frac{1}{M} \sum_{i=1}^{M} (v \cdot \bx_i - v \cdot \mu)^2}
           {\Var_{\bx \sim f} [v \cdot \bx]} 
     \leq 11/10.
\end{equation}
We henceforth assume that both (\ref{e:means}) and (\ref{e:rudelson}) hold (this happens with probability at least $1/2$), and we let $\tmu$ and $x_1,\dots,x_M$ denote the corresponding outcomes.

Let $\Sigma$ be the true co-variance of $f$, 
and let $\tSigma$
be the estimate that was used (which depends on $\tilde{\mu}$).

We have that
\begin{align*}
\frac{v^T \tSigma v}{v^T \Sigma v}
  & = \frac{\frac{1}{M} \sum_{i=1}^{M} (v \cdot x_i - v \cdot \tmu )^2}{v^T \Sigma v} \\
  & = \frac{\frac{1}{M} \sum_{i=1}^{M} (v \cdot x_i - v \cdot \tmu )^2}{\Var_{\bx \sim f} [v \cdot \bx]} \\
  & = \frac{\frac{1}{M} \sum_{i=1}^{M} (v \cdot x_i - v \cdot \mu + v \cdot \mu - v \cdot \tmu)^2}
           {\Var_{\bx \sim f} [v \cdot \bx]} \\
  & \leq \frac{\frac{1}{M} \sum_{i=1}^{M} (v \cdot x_i - v \cdot \mu)^2}
           {\Var_{\bx \sim f} [v \cdot \bx]} 
           + \frac{\frac{2 |v \cdot \mu - v \cdot \tmu|}{M} \sum_{i=1}^{M} |v \cdot x_i - v \cdot \mu|}
           {\Var_{\bx \sim f} [v \cdot \bx]} 
           + \frac{\frac{1}{M} \sum_{i=1}^{M} (v \cdot \mu - v \cdot \tmu)^2}
           {\Var_{\bx \sim f} [v \cdot \bx]} 
                \\
  & = \frac{\frac{1}{M} \sum_{i=1}^{M} (v \cdot x_i - v \cdot \mu)^2}
           {\Var_{\bx \sim f} [v \cdot \bx]} 
           + \frac{2 |v \cdot \mu - v \cdot \tmu|}{M \sqrt{\Var_{\bx \sim f} [v \cdot \bx]}} 
              \sum_{i=1}^{M} \frac{|v \cdot x_i - v \cdot \mu|}
                                    {\sqrt{\Var_{\bx \sim f} [v \cdot \bx]}} 
           + \frac{(v \cdot \mu - v \cdot \tmu)^2}
           {\Var_{\bx \sim f} [v \cdot \bx]} 
                \\
  & \leq \frac{\frac{1}{M} \sum_{i=1}^{M} (v \cdot x_i - v \cdot \mu)^2}
           {\Var_{\bx \sim f} [v \cdot \bx]} 
           + \frac{2 |v \cdot \mu - v \cdot \tmu|}{M \sqrt{\Var_{\bx \sim f} [v \cdot \bx]}} 
              \times \sqrt{M \sum_{i=1}^{M} \frac{(v \cdot x_i - v \cdot \mu)^2}
                                    {\Var_{\bx \sim f} [v \cdot \bx]} }
           + \frac{(v \cdot \mu - v \cdot \tmu)^2}
           {\Var_{\bx \sim f} [v \cdot \bx]} 
                \\
  & = \frac{\frac{1}{M} \sum_{i=1}^{M} (v \cdot x_i - v \cdot \mu)^2}
           {\Var_{\bx \sim f} [v \cdot \bx]} 
           + \frac{2 |v \cdot \mu - v \cdot \tmu|}{\sqrt{\Var_{\bx \sim f} [v \cdot \bx]}} 
              \times \sqrt{\frac{1}{M} \sum_{i=1}^{M} \frac{(v \cdot x_i - v \cdot \mu)^2}
                                    {\Var_{\bx \sim f} [v \cdot \bx]} }
           + \frac{(v \cdot \mu - v \cdot \tmu)^2}
           {\Var_{\bx \sim f} [v \cdot \bx]} 
                \\
  & \leq 11/10 + 2 (1/10) \sqrt{11/10} + 1/100 \leq 2,
\end{align*}
where the second inequality is by Cauchy-Schwarz and the third inequality is
by (\ref{e:rudelson}) and (\ref{e:means}).
Similarly, 
\begin{align*}
\frac{v^T \tSigma v}{v^T \Sigma v}
  & \geq \frac{\frac{1}{M} \sum_{i=1}^{M} (v \cdot x_i - v \cdot \mu)^2}
           {\Var_{\bx \sim f} [v \cdot \bx]} 
           - \frac{\frac{2 |v \cdot \mu - v \cdot \tmu|}{M} \sum_{i=1}^{M} |v \cdot x_i - v \cdot \mu|}
           {\Var_{\bx \sim f} [v \cdot \bx]} 
                \\
  & \geq 9/10 - 2 (1/10) \sqrt{11/10} \geq 1/2,
\end{align*}
completing the proof.
\end{proof}

\medskip

\begin{proofof}{Theorem~\ref{thm:log-concave}}
The basic algorithm (for the noise-free setting) applies the procedure \textsf{rescale} from Lemma~\ref{lem:scaling} to find
an estimate of the covariance matrix of $f$, rescales the axes so that the transformed distribution
is $2$-nearly-isotropic, learns the transformed distribution, and then rescales the axes again to restore their
original scales.

In the presence of noise, Lemma~\ref{lem:scaling} succeeds with probability 1/2, if all the examples
are not noisy.  But since the noise rate is at most $\epsilon$, and we may assume without loss of
generality that $\epsilon < 1/10$, since the number of examples required in Lemma~\ref{lem:scaling}
is independent of $\epsilon$, any invocation of the method succeeds in the presence of noise
with probability $\Omega_d(1)$, which is at least some positive constant (since $d$ is a constant).  
Thus, if an algorithm performs
$O_d(\log(1/\delta))$ many repetitions, with probability at least $1 - \delta/2$ one of them will succeed.
It can therefore call the algorithm of Theorem~\ref{thm:main-semiagnostic}
$O(\log(1/\delta))$ times, and then applying the hypothesis testing procedure of
Proposition~\ref{prop:log-cover-size} to the results, to achieve the claimed result.
\end{proofof}


\section{Learning shift-invariant densities over $\mathbb{R}^d$ with bounded support requires $\Omega(1/\eps^d)$ samples}
\label{sec:lowerbound}

In this section we give a simple lower bound which shows that\ignore{the sample complexity of the algorithm of Lemma~\ref{lem:finite-support} is information-theoretically optimal:}  $\Omega(1/\eps^d)$ samples are required for $\eps$-accurate density estimation even of shift-invariant $d$-dimensional densities with bounded support.  As discussed in the introduction, densities with bounded support may be viewed as satisfying the strongest possible rate of tail decay as they have zero tail mass outside of a bounded region.

\begin{theorem} \label{thm:lb}
Given $d \geq 1$, there is a constant $c_d=\Theta(\sqrt{d})$ such that the following holds:  For all sufficiently small $\eps$,
let $A$ be an algorithm with the following property:  given access to $m$ i.i.d. samples from an arbitrary (and unknown) finitely supported density $f \in {\cal C}_\si(c_d,d)$, with probability at least $99/100$, $A$ outputs a hypothesis density $h$ such that $\dtv(f,h) \leq \eps.$  Then $m \geq \Omega((1/\eps)^d)$.
\end{theorem}

Since an algorithm that achieves a small error with high probability can be used to achieve small error in expectation, to prove Theorem~\ref{thm:lb} it suffices to show that any algorithm that achieves expected error $O(\eps)$ must use
$\Omega((1/\eps)^d)$ samples.  To establish this we use  Lemma~\ref{fano} (given below), which provides a lower bound on the number of examples needed for small expected error.  

To obtain the desired lower bound from Lemma~\ref{fano}, we establish the existence of a family ${\cal F}$ of densities ${\cal F} = \{f_1,\dots,f_N\} \in {\cal C}_\si(c_d,d)$, where $N= \exp(\Omega((1/\eps)^d))$.  These densities will be shown to satisfy the following two properties: for any  $i \neq j \in [N]$ we have (1) $\dtv(f_i,f_j) = \Omega(\eps)$, and (2) the Kullback-Leibler divergence $D_{KL}(f_i || f_j)$ is at most $O(1)$, yielding Theorem~\ref{thm:lb}. 

\subsection{Fano's inequality}

The main tool we use for our lower bound is Fano's inequality, or more precisely, the following extension of it
given by 
\citep{IK79} and \citep{AB83}:

\begin{theorem}[Generalization of Fano's Inequality.] \label{fano}
Let $f_1,\dots,f_{N+1}$ be a collection of $N+1$ distributions such that for any $i \neq j \in [N+1]$, we
have (i) $\dtv(f_i,f_j) \geq \alpha/2$, and (ii) $D_{KL}(f_i || f_j) \leq \beta$, where $D_{KL}$ denotes Kullback-Leibler divergence.  Then for any algorithm that makes $m$ draws from an unknown target distribution $f_i$, $i \in [N+1]$, and outputs a hypothesis distribution $\tilde{f}$, there is some $i \in [t+1]$ such that if the target distribution is $f_i$, then 
\[
\E[\dtv(f,\tilde{f})] \geq {\frac \alpha 2} \left(1 - {\frac {m \beta + \ln 2}{\ln N}}\right).
\]
In particular, to achieve expected error at most $\alpha/4$, any learning algorithm must have sample complexity
$m=\Omega \left(\frac{\ln N}{\beta}\right).$
\end{theorem}

\subsection{The family of densities we analyze} \label{sec:family}

Let $T$ be a positive integer that is $T=\lceil C/\epsilon\rceil$ for a large constant $C$.
We consider probability densities over $\mathbb{R}^d$ which  (i) are supported on $[-T,T)^d,$ and (ii) are piecewise constant on each of the $(2T)^d$ many disjoint unit cubes whose union is $[-T,T)^d.$  Writing $A$ to denote the set $\{-T,-T+1,\dots,-1,0,1,\dots,T-1\}^d,$ each of the $(2T)^d$ many disjoint unit cubes mentioned above is indexed by a unique element $a=(a_1,\dots,a_d) \in A$ in the obvious way.  We write $\cube(a)$ to denote the unit cube indexed by $a$.  Given $x \in [-T,T)^d$ we write $a(x)$ to denote the unique element $a \in A$ such that $x \in \cube(a).$

For any $z \in \{0,1\}^{A}$, we define a probability density $f_z$ over $[-T,T)^d$ as $f_z(x) = (T+z_{a(x)})/Z,$
\ignore{\[f_z(x) = \begin{cases}
(T+1)/Z 		& \text{~if~}z_{a(x)}=1,\\
T/Z	& \text{~if~}z_{a(x)}=0,
\end{cases}
\]
}where $Z=\Theta((2T)^{(d+1)})$ is a normalizing factor so that $f_z$ is indeed a density (i.e. it integrates to 1).

It is well known (via an elementary probabilistic argument) that there is a subset $S \subset \{0,1\}^A$ of size $2^{\Theta(|A|)}$ such that any two distinct strings $z,z' \in S$ differ in $\Theta(|A|)$ many coordinates.  We define the set ${\cal F}$ of densities to be ${\cal F} = \{f_z: z \in S\}.$

\subsection{Membership in ${\cal C}_\si(c_d,d)$}

It is obvious that every density in $\cF$ is finitely supported. In this subsection we prove that $\cF \ss \cC_\si(c_d,d)$.  
First, we bound the variation distance incurred by shifting along a coordinate axis:
\begin{lemma}
\label{l:coord.shift}
For any $f \in \cF$, $i \in [d]$, and $\kappa \in (0,1)$, we have
$
\int | f(x + \kappa e_i) - f(x) | \;dx \leq \Theta(\kappa \eps).
$
\end{lemma}
\begin{proof}
We have
\begin{align*}
\int | f(x + \kappa e_i) - f(x) | \;dx  
& = \int_{\{ x: x_i < -T \}} | f(x + \kappa e_i) - f(x) | \;dx   
 + \int_{\{ x: x_i > T-\kappa \}} | f(x + \kappa e_i) - f(x) | \;dx   \\
& \hspace{0.5in} + \int_{\{ x: x_i \in [-T, T-\kappa]\}} | f(x + \kappa e_i) - f(x) | \;dx.
\end{align*}
If $x_i < -T-\kappa$, then $f(x + \kappa e_i) = f(x) = 0$.  When
${-T-\kappa  \leq x_i  < -T}$, we have $| f(x + \kappa e_i) - f(x) | \leq (T+1)/Z$.  Thus
\[
\int_{\{ x: x_i < -T \}} | f(x + \kappa e_i) - f(x) | \;dx  \leq (\kappa (2T)^{d-1}) (T+1)/Z = {\Theta(\kappa \eps)}.
\]
A similar argument gives that
$
\int_{\{ x: x_i > T-\kappa \}} | f(x + \kappa e_i) - f(x) | \;dx \leq \Theta(\kappa \eps).$
Finally,
\begin{align*} \int_{\{ x: x_i \in [-T, T-\kappa]\}} | f(x + \kappa e_i) - f(x) | \;dx 
& = \int_{\{ x: x_i \in [-T, T-\kappa], \lceil x_i \rceil - x_i \leq \kappa\}} | f(x + \kappa e_i) - f(x) | \;dx \\
& \leq \int_{\{ x: x_i \in [-T, T-\kappa], \lceil x_i \rceil - x_i \leq \kappa\}} (1/Z) \;dx.
\end{align*}
The set $\{ x: x_i \in [-T, T-\kappa], \lceil x_i \rceil - x_i \leq \kappa\}$ is made up of $2T-1$ ``slabs'' that are each of width $\kappa$, and consequently
$\int_{\{ x: x_i \in [-T, T-\kappa], \lceil x_i \rceil - x_i \leq \kappa\}} \leq (2 T - 1) \kappa (2T)^{d-1}/Z = \Theta(\kappa/T)
{=\Theta(\kappa \eps)}$,
recalling that $Z = \Theta((2T)^{d+1})$.  This completes the proof.
\end{proof} 

\ignore{
}

Given Lemma~\ref{l:coord.shift}, it is easy to bound the variation distance incurred by shifting in an arbitrary direction:
\begin{lemma}
\label{l:unit.length.shift}
For any $f \in \cF$, for any unit vector $v$, for any $\kappa < 1$, we have
$
\int | f(x + \kappa v) - f(x) | \;dx \leq c \kappa \eps \sqrt{d}
$
for a universal constant $c>0.$
\end{lemma}
\begin{proof}
Writing the unit vector $v$ as $\sum_{j=1}^d v_j e_j$, we have
\begin{align*}
\int | f(x + \kappa v) - f(x) | \;dx 
& = \int \left| \sum_{i=1}^d 
          f\left(x + \kappa \sum_{j=1}^i v_j e_j \right) - f\left(x + \kappa \sum_{j=1}^{i-1} v_j e_j \right) \right| \;dx  \\
& \leq \sum_{i=1}^d \int \left| f\left(x + \kappa \sum_{j=1}^i v_j e_j\right) - f\left(x + \kappa \sum_{j=1}^{i-1} v_je_j\right) \right| \;dx \tag{triangle inequality}\\
& = \sum_{i=1}^d \int | f\left(x + \kappa |v_i| e_i\right) - f\left(x \right) | \;dx \tag{variable substitution} \\
& \leq {c \kappa \eps} \sum_{i=1}^d | v_i | \leq c \kappa \eps \sqrt{d}, \tag{Lemma~\ref{l:coord.shift} and Cauchy-Schwarz}\\
\end{align*}
completing the proof.
\end{proof}

As an easy consequence of Lemma~\ref{l:unit.length.shift}, Lemma~\ref{l:var} and the definition of $\si(f,v,\kappa)$ we obtain the following:

\begin{corollary} \label{cor:si}
There is a constant $c_d{=\Theta(\sqrt{d})}$ such that for any $f \in {\cal F}$ and any unit vector $v\in \mathbb{R}^d$, we have $\si(f,v) \leq
c_d.$  Hence $\cF \ss \cC_\si(c_d,d)$.
\end{corollary}

\subsection{The upper bound on $KL$ divergence and lower bound on variation distance}

Recall that if $f$ and $g$ are probability density functions supported on a set $S \subseteq \mathbb{R}^d$, then the \emph{Kullback Leibler divergence} between $f$ and $g$ is defined as
$
D_{KL}(f || g) = \int_S f(x) \ln {\frac {f(x)}{g(x)}} dx.
$
As an immediate consequence of this definition, we have the following claim:

\begin{claim} Let $f,g$ be two densities such that for some absolute constant $C > 1$ we have that every $x$  satisfies ${\frac 1 C} f(x) \leq g(x) \leq C f(x).$  Then $D_{KL}(f || g) \leq O(1).$
\end{claim}

It is easy to see that any $f_i,f_j$ in the family ${\cal F}$ of densities described in Section~\ref{sec:family} are such that 
${\frac 1 C} f_i(x) \leq f_j(x) \leq C f_i(x).$  Thus we have:

\begin{lemma} \label{lem:KL}
$D_{KL}(f_i || f_j) \leq O(1)$ for all $i \neq j \in [N]$.
\end{lemma}

Finally, we need a lower bound on the total variation distance between any pair of elements of $\cF$:
\begin{lemma}
\label{l:tv}
There is an absolute constant $c>0$ such that,
for any $f_u, f_v \in \cF$, $\dtv(f_u, f_v) = \Omega(\eps)$.
\end{lemma}
\begin{proof}
We have $
\dtv (f_u, f_v) = (1/Z) |\{ a \in A : u_a \neq v_a \}| = (1/Z) \Omega(|A|) = \Omega(1/T) = \Omega(\eps).
$
\end{proof}

\subsection{Putting it together}

By Lemma~\ref{l:tv}, each pair of elements of $\cF$ are separated by $\Omega(\epsilon)$ in
total variation distance.  Since Lemma~\ref{lem:KL} implies that each pair of elements of $\cF$ has KL-divergence $O(1)$, Theorem~\ref{fano} implies that $\Omega(\ln |{\cal F}|) = \Omega((1/\eps)^d)$ examples are needed to achieve expected error at most $O(\eps).$  Since Corollary~\ref{cor:si} gives that $\cF \ss \cC_\si(c_d,d)$ for a constant $c_d=\Theta(\sqrt{d})$, this proves
Theorem~\ref{thm:lb}.

\bibliographystyle{alpha}
\bibliography{allrefs}

\appendix

%


\end{document}